%% file: arxiv.tex
\documentclass[nohyperref]{article}

\usepackage{microtype}
\usepackage{graphicx}
\usepackage{subfigure}
\usepackage{booktabs} 
\usepackage{subfigure}

\usepackage{hyperref}

\usepackage[preprint]{icml2023}

\usepackage{amsmath}
\usepackage{amssymb}
\usepackage{mathtools}
\usepackage{amsthm}

\usepackage[capitalize,noabbrev]{cleveref}

\theoremstyle{plain}
\newtheorem{theorem}{Theorem}[section]

\newtheorem{lemma}[theorem]{Lemma}

\theoremstyle{definition}

\theoremstyle{remark}

\usepackage[textsize=tiny]{todonotes}

\icmltitlerunning{Provable Robustness for Streaming Models with a Sliding Window} 

\usepackage{csquotes}
\usepackage{wrapfig}

\newtheorem*{statement}{\bf Statement}
\DeclareMathOperator\TV{\mathsf{TV}}
\DeclareMathOperator\erf{erf}

\begin{document}

\twocolumn[
\icmltitle{Provable Robustness for Streaming Models with a Sliding Window}




\begin{icmlauthorlist}
\icmlauthor{Aounon Kumar}{umd}
\icmlauthor{Vinu Sankar Sadasivan}{umd}
\icmlauthor{Soheil Feizi}{umd}
\end{icmlauthorlist}

\icmlaffiliation{umd}{Department of Computer Science, University of Maryland -- College Park, MD, USA}

\icmlcorrespondingauthor{Aounon Kumar}{aounon@umd.edu}
\icmlcorrespondingauthor{Soheil Feizi}{sfeizi@cs.umd.edu}

\icmlkeywords{Machine Learning, ICML}

\vskip 0.3in
]



\printAffiliationsAndNotice{}  

\begin{abstract}
The literature on provable robustness in machine learning has primarily focused on static prediction problems, such as image classification, in which input samples are assumed to be independent and model performance is measured as an expectation over the input distribution.
Robustness certificates are derived for individual input instances with the assumption that the model is evaluated on each instance separately.
However, in many deep learning applications such as online content recommendation and stock market analysis, models use historical data to make predictions.
Robustness certificates based on the assumption of independent input samples are not directly applicable in such scenarios.
In this work, we focus on the provable robustness of machine learning models in the context of data streams, where inputs are presented as a sequence of potentially correlated items.
We derive robustness certificates for models that use a fixed-size sliding window over the input stream.
Our guarantees hold for the average model performance across the entire stream and are independent of stream size, making them suitable for large data streams.
We perform experiments on speech detection and human activity recognition tasks and show that our certificates can produce meaningful performance guarantees against adversarial perturbations.

\end{abstract}

\section{Introduction}

Deep neural network (DNN) models are increasingly being adopted for real-time decision-making and prediction tasks.
Once a neural network is trained, it is often required to make fast predictions on an evolving stream of inputs, as in algorithmic trading \citep{Zhang2017trading, KRAUSS2017689, Korczak8104660, FISCHER2018654, ozbayoglu2020deep}, human action recognition \citep{Yang2015HAR, lstm2016HAR, RONAO2016235} and speech detection \citep{GRAVES2005602, Dennis2019SRNN, hsiao2020online}.
However, despite their impressive performance, 
DNNs are known to malfunction under tiny perturbations of the input, designed to fool them into making incorrect predictions \citep{Szegedy2014, BiggioCMNSLGR13, GoodfellowSS14, MadryMSTV18,  Carlini017}.
This vulnerability is not limited to static models like image classifiers and has been demonstrated for streaming models as well \citep{braverman2021adversarial, mladenovic2022online, EliezerJWY20, EliezerY20_sampling}.
Such input corruptions, commonly known as adversarial attacks, make DNNs especially risky for safety-critical applications of streaming models such as health monitoring \citep{IGNATOV2018915, Stamate7917848, lee2019predicting, cai2020review} and autonomous driving \citep{BojarskiTDFFGJM16, Xu2017driving, janai2020computer}.
What makes the adversarial streaming setting more challenging than the static one is that the adversary can exploit historical data to strengthen its attack.
For instance, it could wait for a critical decision-making point, such as a trading algorithm making a buy/sell recommendation or an autonomous vehicle approaching a stop sign, before generating an adversarial perturbation.

Over the years, a long line of research has been dedicated to mitigating the adversarial vulnerabilities of DNNs.
These methods seek to improve the empirical robustness of a model by introducing input corruptions during training \citep{KurakinGB17, BuckmanRRG18, GuoRCM18, DhillonALBKKA18, LiL17, GrosseMP0M17, GongWK17}.
However, such empirical defenses have been shown to break down under stronger adversarial attacks \citep{Carlini017, athalye18a, UesatoOKO18, tramer2020adaptive}.
This motivated the study of provable robustness in machine learning which seeks to obtain verifiable guarantees on the predictive performance of a DNN.
Several certified defense techniques have been developed over the years, most notable of which are based on convex relaxation \citep{WongK18, Raghunathan2018, Singla2019, Chiang20, singla2020secondorder}, interval-bound propagation \citep{gowal2018effectiveness, HuangSWDYGDK19, dvijotham2018training, mirman18b} and randomized smoothing \citep{cohen19, LecuyerAG0J19, LiCWC19, SalmanLRZZBY19, Levine-ICML21}.
Most research in provable robustness has focused on static prediction tasks like image classification and the streaming machine learning (ML) setting has not yet been considered.

In this work, we derive provable robustness guarantees for the streaming setting where inputs are presented as a sequence of potentially correlated items.
Our objective is to design robustness certificates that produce guarantees on the average model performance over long, potentially infinite, data streams.
Our threat model is defined as a man-in-the-middle adversary present between the DNN and the data stream that can perturb the input items before they are passed to the DNN.
The adversary is constrained by a limit on the average perturbation added to the inputs.
We show that a DNN that randomizes the inputs before making predictions is guaranteed to achieve a certain performance level for any adversary within the threat model.
Unlike many randomized smoothing-based approaches that aggregate predictions over several noised samples ($\sim 10^6$) of the input instance, our procedure only requires one sample of the randomized input, keeping the computational complexity of the DNN unchanged.
Our certificates are independent of the stream length, making them suitable for large streams.

\textbf{Technical Challenges:}
Provable robustness procedures developed for static tasks like image classification assume that the inputs are sampled independently from the data distribution.
Robustness certificates are derived for individual input instances with the assumption that the DNN is applied to each instance separately and the adversarial perturbation added to one instance does not affect the DNN's predictions on another.
However, in the streaming ML setting, the prediction at a given time-step is dependent on past input items in the data stream and a worst-case adversary can exploit this dependence between inputs to adapt its strategy and strengthen its attack.
A robustness certificate that is derived based on the assumption of independence of input samples may not hold for such correlated inputs.
Thus, there is a need to design provable robustness techniques tailored specifically for the streaming ML setting.

Out of the existing certified robustness techniques, randomized smoothing has become prominent due to its model-agnostic nature, scalability for high-dimensional problems \citep{LecuyerAG0J19}, and flexibility to adapt to different machine learning paradigms like reinforcement learning and structured outputs \citep{PolicySmoothing, wu2021crop, kumar2021center}.
This makes randomized smoothing a suitable candidate for provable robustness in streaming ML.
However, conventional randomized smoothing approaches require several evaluations ($\sim 10^6$) of the prediction model on different noise vectors in order to produce a robust output.
This significantly increases the computational requirements of the model making them infeasible for real-world streaming applications which require decisions to be made in a short time frame such as high-frequency trading and autonomous driving.
Our goal is to obtain robustness guarantees for a simple technique that only adds a single noise vector to the DNN's input.

Existing works on provable robustness in reinforcement learning \citep{PolicySmoothing, wu2021crop} indicate that if the prediction at a given time-step is a function of the entire stream till that step, the robustness guarantees worsen with the length of the stream and become vacuous for large stream sizes.
The tightness analysis of these certificates suggests that it might be difficult to achieve robustness guarantees that are independent of the stream size.
However, many practical streaming models use only a bounded number of past input items in order to make predictions at a given time step.
Recent work by \citet{EfroniJKM22} has also shown that near-optimal performance can be achieved by only observing a small number of past inputs for several real-world sequential decision-making problems. 
This raises the natural question:
\begin{displayquote}
Can we obtain better certificates if the DNN only used a fixed number of inputs from the stream?
\end{displayquote}

\begin{figure}[t]
    \centering
    \includegraphics[width=1.05\columnwidth]{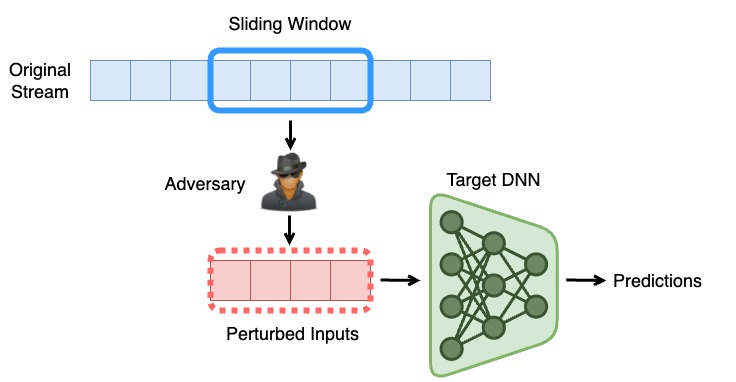}
    \vspace{-3mm}
    \caption{Adversarial Streaming Threat Model.}
    \label{fig:threat_model}
\end{figure}

\textbf{Our Contributions:} We design a robustness certificate for streaming models that use a fixed-sized sliding window over the data stream to make predictions (see Figure~\ref{fig:threat_model}).
In our setting, the DNN only uses the part of the data stream inside the window at any given time step.
We certify the average performance $Z$ of the model over a stream of size $t$:
\[Z = \frac{\sum_{i=1}^t f_i}{t},\]
where each $f_i$ measures the predictive performance of the DNN at time-step $i$ as a value in the range $[0, 1]$.

The adversary is allowed to perturb the input items inside the window at every time step separately.
The strength of the adversary is limited by a bound $\epsilon$ on the average size of the perturbation added:
\begin{equation*}
    \frac{\sum_{i=1}^t \sum_{k = 1}^w d(x_i, x_i^k)}{wt} \leq \epsilon,
\end{equation*}
where $x_i$ and $x_i^k$ are the input item at time-step $i$ and its $k$th adversarial perturbation respectively, $w$ is the window size and $d$ is a distance function to measure the size of the adversarial perturbations, e.g., $d(x_i, x_i^k) = \|x_i - x_i^k\|_2$.
Our adversarial threat model is general enough to subsume the scenario where the attacker only perturbs each stream element only once as a special case where all $x_i^k$s are set to some $x_i'$.

Our main theoretical result shows that the difference between the clean performance of a robust streaming model $\tilde{Z}$ and that in the presence of an adversarial attack $\tilde{Z}_\epsilon$ is bounded as follows:
\begin{equation}
\label{eq:main_result}
|\tilde{Z} - \tilde{Z}_\epsilon | \leq w \psi(\epsilon),
\end{equation}
where $\psi(.)$ is a concave function that bounds the total variation between the smoothing distributions at two input points as a function of the distance between them (condition~(\ref{eq:tv_bound}) in Section~\ref{sec:notations}).
Such an upper bound always exists for any smoothing distribution.
For example, when the distance between the points is measured using the $\ell_2$-norm and the smoothing distribution is a Gaussian $\mathcal{N}(0, \sigma^2 I)$ with variance $\sigma^2$, then the concave upper bound is given by $\psi(\cdot) = \erf(\cdot/2\sqrt{2}\sigma)$.
Our robustness certificate is independent of the length of the stream and depends only on the window size $w$ and average perturbation size $\epsilon$.
This suggests that streaming ML models with smaller window sizes are provably more robust to adversarial attacks.

We perform experiments on two real-world applications -- human activity recognition and speech keyword detection. We  use the UCI HAR datset \citep{reyes2012uci} for human activity  recognition and the Speech commands dataset \citep{speechcommandsv2} for speech keyword detection. We train convolutional networks that take sliding windows as inputs and provide robustness guarantees for their performance. In our experiments, we consider two different scenarios for the adversary. In the first case, the adversary can perturb an input only once. In the more general second scenario, the adversary can perturb each sliding window separately, making it a powerful attacker. We develop strong adversaries for both of these scenarios and show their effectiveness in our experiments. We then show that our certificates provide meaningful robustness guarantees in the presence of such strong adversaries. Consistent with our theory, our experiments also demonstrate that a smaller window size $w$ gives a stronger certificate.



\section{Related Work}
The adversarial streaming setup has been studied extensively in recent years.
\citet{mladenovic2022online} designed an attack for transient data streams that do not allow the adversary to re-attack past input items.
In their setting, the adversary only has partial knowledge of the target DNN and the perturbations applied in previous time steps are irrevocable.
Their objective is to produce an adversarial attack with minimal access to the data stream and the target model.
Our goal, on the other hand, is to design a provably robust method that can defend against as general and as strong an adversary as possible.
We assume that the adversary has full knowledge of the parameters of the target DNN and can change the adversarial perturbations added in previous time steps.
Our threat model includes transient data streams as a special case and applies even to adversaries that only have partial access to the DNN.
Streaming adversarial attacks have also been studied for sampling algorithms such as Bernoulli sampling and reservoir sampling \citep{EliezerY20_sampling}.
Here, the goal of the adversary is to create a stream that is unrepresentative of the actual data distribution.
Other works have studied the adversarial streaming setup for specific data analysis problems like frequency moment estimation \citep{EliezerJWY20}, submodular maximization \citep{MitrovicBNTC17}, coreset construction and row sampling \citep{braverman2021adversarial}.
In this work, we focus on a robustness certificate for general DNN models in the streaming setting under the conventional notion of adversarial attacks in machine learning literature.
We use a sliding-window computational model which has been extensively studied over several years for many streaming applications \cite{GanardiHL19, FeigenbaumKZ04, Datar2007}.
Recently \citet{EfroniJKM22} also showed that a short-term memory is sufficient for several real-world reinforcement learning tasks.

\vspace{-1mm}
A closely related setting is that of adversarial reinforcement learning.
Adversarial attacks have been designed that either directly corrupt the observations of the agent \citep{HuangPGDA17, BehzadanM17, PattanaikTLBC18} or introduce adversarial behavior in a competing agent \cite{GleaveDWKLR20}.
Robust training methods, such as adding adversarial noise \citep{Kamalaruban2020, Vinitsky2020} and training with a learned adversary in an online alternating fashion \citep{zhang2021robust}, have been proposed to in improve the robustness of RL agents.
Several certified defenses have also been developed over the years.
For instance, \citet{Zhang-NeurIPS2020} developed a method that can certify the actions of an RL agent at each time step under a fixed adversarial perturbation budget.
It can certify the total reward obtained at the end of an episode if each of the intermediate actions is certifiably robust.
Our streaming formulation allows the adversary to choose the budget at each time step as long as the average perturbation size remains below $\epsilon$ over time.
Our framework also does not require each prediction to be robust in order to certify the average performance of the DNN.
More recent works in certified RL can produce robustness guarantees on the total reward without requiring every intermediate action to be robust or the adversarial budget to be fixed \citep{PolicySmoothing, wu2021crop}.
However, these certificates degrade for longer streams and the tightness analysis of these certificates indicates that this dependence on stream size may not be improved.
Our goal is to keep the robustness guarantees independent of stream size so that they are suitable even for large streams.

The literature on provable robustness has primarily focused on static prediction problems like image classification.
One of the most prominent techniques in this line of research is randomized smoothing.
For a given input image, this technique aggregates the output of a DNN on several noisy versions of the image to produce a robust class label \citep{LecuyerAG0J19, cohen19}.
This is the first approach that scaled up to high-dimensional image datasets like ImageNet for $\ell_2$-norm bounded adversaries..
It does not make any assumptions on the underlying neural network such as Lipschitz continuity or a specific architecture, making it suitable for conventional DNNs that are several layers deep.
However, randomized smoothing also suffers some fundamental limitations for higher norms such as the $\ell_\infty$-norm \citep{kumar2020curse}.
Due to its flexible nature, randomized smoothing has also been adapted for tasks beyond classification, such as segmentation and deep generative modeling, with multi-dimensional and structured outputs like images, segmentation masks, and language \cite{kumar2021center}.
For such outputs, robustness certificates are designed in terms of a distance metric in the output space such as LPIPS distance, intersection-over-union and total variation distance.
However, provable robustness in the static setting assumes a fixed budget on the size of the adversarial perturbation for each input instance and does not allow the adversary to choose a different budget for each instance.
In our streaming threat model, we allow the adversary the flexibility of allocating the adversarial budget to different time steps in an effective way, attacking more critical input items with a higher budget and conserving its budget at other time steps.
Recent work on provable robustness against Wasserstein shifts of the data distribution allows the adversary to choose the attack budget for each instance differently \citep{kumar2022dist_shift}.
However, unlike our streaming setting, the input instances are drawn independently from the data distribution and the adversarial perturbation applied to one instance does not impact the performance of the DNN on another.



\section{Preliminaries and Notation}
\label{sec:notations}
\textbf{Streaming ML Setting:} We define a data stream of size $t$ as a sequence of input items $x_1, x_2, \ldots, x_i, \ldots, x_t$ generated one-by-one from an input space $\mathcal{X}$ over discrete time steps.
At each time step $i$, a DNN model $\mu$ makes a prediction that may depend on no more than $w$ of the previous inputs.
We refer to the contiguous block of past input items as a window $W_i \in \mathcal{X}^{\min(i, w)}$ of size $w$ defined as follows:
\begin{align*}
    W_i = 
    \begin{cases}
    (x_1, x_2, \ldots, x_i) & \text{for } i \leq w\\
    (x_{i-w+1}, x_{i-w+2}, \ldots, x_i) & \text{otherwise.}
    \end{cases}
\end{align*}

The performance of the model $\mu$ at time step $i$ is given by a function $f_i: \mathcal{X}^{\min(i, w)} \rightarrow [0, 1]$ that passes the window $W_i$ through the model $\mu$, compares the prediction with the ground truth and outputs a value in the range $[0, 1]$.
For instance, in speech recognition, the window $W_i$ would represent the audio from the past few seconds which gets fed to the model $\mu$.
The function $f_i = \mathbf{1}\{\mu(W_i) = y_i \}$ could indicate whether the prediction of $\mu$ matches the ground truth $y_i$.
Similarly, in autonomous driving, we can define a performance function $f_i = \text{IoU}(\mu(W_i), y_i)$ that measures the average intersection-over-union of the segmentation mask of the surrounding environment.
We define the overall performance $Z$ of the model $\mu$ as an average over the $t$ time-steps:
\[Z = \frac{\sum_{i=1}^t f_i}{t}.\]

\textbf{Threat Model:} An adversary $A$ is present between the DNN and the data stream which can perturb the inputs with the objective of minimizing the average performance $Z$ of the DNN (see Figure~\ref{fig:threat_model}).
Let $x_i'$ be the perturbed input at step $i$.
We define a constraint on the amount by which the adversary can perturb the inputs as a bound on the average distance between the original input items $x_i$ and their perturbed versions $x_i'$:
\begin{equation}
\label{eq:adv_constraint}
\frac{\sum_{i=1}^t d(x_i, x_i')}{t} \leq \epsilon,
\end{equation}
where $d$ is a function that measures the distance between a pair of input items from $\mathcal{X}$, e.g., $d(x_i, x_i') = \|x_i - x_i'\|_2$.
The adversary seeks to minimize the overall performance $Z$ of the model without violating the above constraint, i.e.,
\[\min_{A \in \mathcal{A}_{\epsilon}} \sum_{i=1}^t f_i(A(x_i), A(x_{i-1}), \ldots, A(x_{i-w+1})) / t,\]
where $\mathcal{A}_{\epsilon}$ is the set of all adversaries satisfying constraint~(\ref{eq:adv_constraint}).
We also study another threat model where the adversary is allowed to attack an input item $x_i$ in every window that it appears in.
We denote the $k$-th attack of $x_i$ as $x_i^k$ and redefine the above constraint as follows:
\begin{equation}
\label{eq:adv_constraint_window}
    \frac{\sum_{i=1}^t \sum_{k = 1}^w d(x_i, x_i^k)}{wt} \leq \epsilon
\end{equation}
This threat model is more general than the one defined by constraint~(\ref{eq:adv_constraint}) because it subsumes this constraint as a special case when all $x_i^k$ are equal to $x_i'$.
Thus, any robustness guarantee that holds for this stronger threat model must also hold for the previous one.

\textbf{Robustness Procedure:} Our goal is to design a procedure that has provable robustness guarantees against the above threat models.
We define a robust prediction model $\tilde{\mu}$: Given an input $x_i \in \mathcal{X}$, we sample a point $\tilde{x}_i$ from a probability distribution $\mathcal{S}(x_i)$ around $x_i$ (e.g., $\mathcal{N}(x_i, \sigma^2 I)$) and evaluate the model $\mu$ on $\tilde{x}_i$.
Define the performance of $\tilde{\mu}$ at time-step $i$ to be the expected value of $f_i$ under the randomized inputs, i.e.,
\[\tilde{f}_i = \mathbb{E}_{\tilde{x_i} \sim \mathcal{S}(x_i)}[f_i(\tilde{x}_i, \tilde{x}_{i-1}, \ldots, \tilde{x}_{i-w+1})]\]
and the overall performance as $\tilde{Z} = \sum_{i=1}^t \tilde{f}_i/t$.

Let $\psi(\cdot)$ be a concave function bounding the total variation between the distributions $\mathcal{S}(x_i)$ and $\mathcal{S}(x_i')$ as a function of the distance between them, i.e.,
\begin{equation}
\label{eq:tv_bound}
\mathsf{TV}(\mathcal{S}(x_i), \mathcal{S}(x_i')) \leq \psi(d(x_i, x_i')).
\end{equation}
Such a bound always exists regardless of the shape of the  smoothing distribution because as the distance between the points $x_i$ and $x_i'$ goes from 0 to $\infty$, the total variation goes from 0 to 1.
A trivial concave bound could be obtained by simply taking the convex hull of the region under the total variation curve (see Figure~\ref{fig:tv_bound}).
However, to find a closed-form expression for $\psi$, we need to analyze different smoothing distributions and distance functions separately.
If the smoothing distribution is a Gaussian $\mathcal{N}(0, \sigma^2 I)$ with variance $\sigma^2$ and the distance is measured using the $\ell_2$-norm, as in all of our experiments, then $\psi(\|x_i - x_i'\|_2) = \erf(\|x_i - x_i'\|_2/2\sqrt{2}\sigma)$, where $\erf$ is the Gauss error function.
For a uniform smoothing distribution within an interval of size $b$ in each dimension of $x_i$ and the $\ell_1$-distance metric, $\psi(\|x_i - x_i'\|_1) = \|x_i - x_i'\|_1 / b$.
See Appendix~\ref{sec:psi-functions} for proof.

\begin{figure}[t]
    \hspace{-3mm}
    \includegraphics[width=1.05\columnwidth]{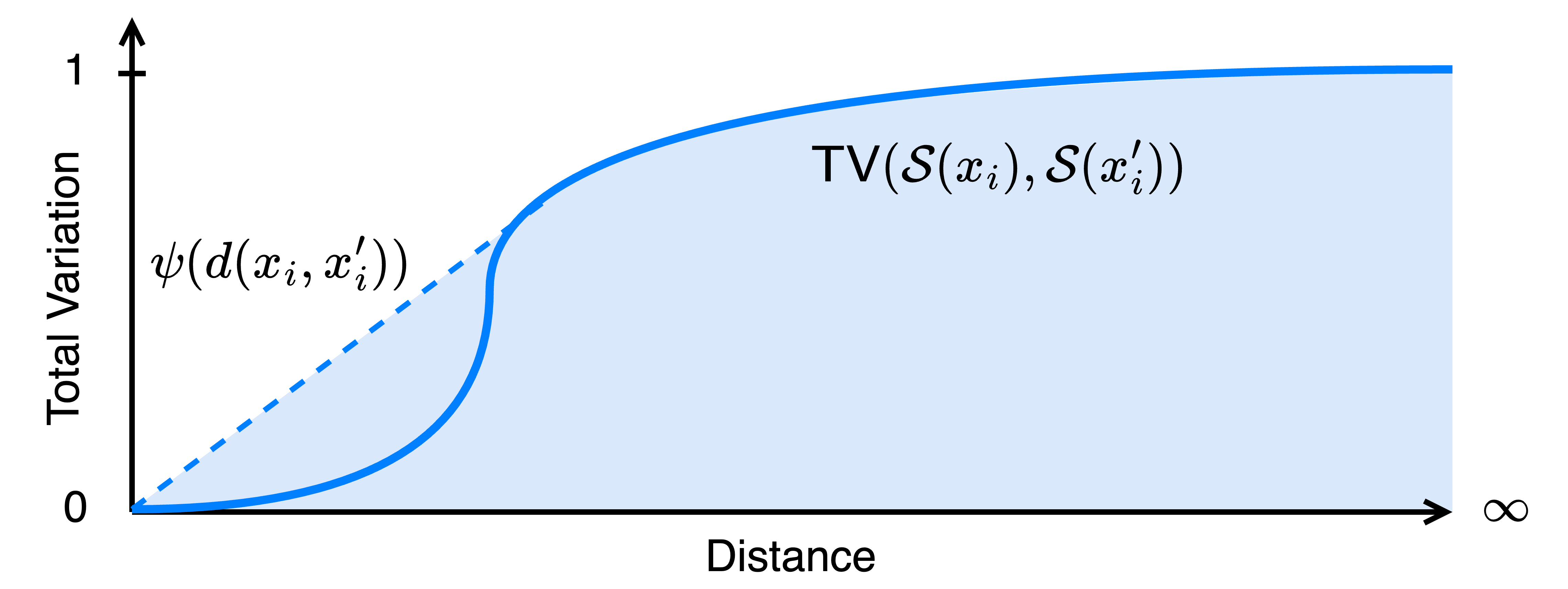}
    \caption{Constructing a concave upper bound $\psi(\cdot)$ for any smoothing distribution $\mathcal{S}$.}
    \label{fig:tv_bound}
\end{figure}
    
\section{Robustness Certificate}
\label{sec:rob_cert}
In this section, we prove robustness guarantees for the simpler threat model defined by constraint~(\ref{eq:adv_constraint}) where each input item is allowed to be attacked only once.
We include complete proofs of our theorems for this threat model in this section for clarity.
The proofs for the more general case in the next section use similar techniques and have been included in the appendix.
In the following lemma, we bound the change in the performance function $\tilde{f}_i$ at each time-step $i$ using the function $\psi$ and the size of the adversarial perturbation added at each step.
For the proof, we first decompose the change in the value of this function into components for each input item.
Since each of these components can be expressed as the difference of the expected value of a function in the range $[0,1]$ under two probability distributions, they can be bounded by the total variation of these distributions.
\begin{lemma}
\label{lem:per-step-bound}
The change in each $\tilde{f}_i$ under an adversary in $\mathcal{A}_\epsilon$ is bounded as
\begin{align*}
|\tilde{f}_i(x_i, x_{i-1}, \ldots, x_{i-s+1}) & - \tilde{f}_i(x_i', x_{i-1}', \ldots, x_{i-s+1}')|\\
    & \leq \sum_{j = i}^{i-s+1} \psi(d(x_j, x_j')),
\end{align*}
where $s = \min(i, w)$.
\end{lemma}

\begin{proof}
The left-hand side of the above inequality can be re-written as:
\begin{align*}
    &|\tilde{f}_i(x_i, x_{i-1}, \ldots, x_{i-s+1}) - \tilde{f}_i(x_i', x_{i-1}', \ldots, x_{i-s+1}')|\\
    &= |\tilde{f}_i(x_i, x_{i-1}, \ldots, x_{i-s+1}) - \tilde{f}_i(x_i', x_{i-1}, \ldots, x_{i-s+1})\\
    & \quad +\tilde{f}_i(x_i', x_{i-1}, \ldots, x_{i-s+1}) - \tilde{f}_i(x_i', x_{i-1}', \ldots, x_{i-s+1}')|\\
    &= \Bigg| \sum_{j = i}^{i-s+1} \tilde{f}_i(x_i', \ldots x_j, \ldots, x_{i-s+1})\\
    & \quad \quad \quad \quad - \tilde{f}_i( x_i', \ldots, x_j', \ldots, x_{i-s+1}) \Bigg|\\
    &\leq \sum_{j = i}^{i-s+1} \Big|\tilde{f}_i(x_i', \ldots x_j, \ldots, x_{i-s+1})\\
    & \quad \quad \quad \quad - \tilde{f}_i(x_i', \ldots, x_j', \ldots, x_{i-s+1})\Big|
\end{align*}
The two terms in each summand differ only in the $j$th input.
Thus, the $j$th term in the above summation can be written as the difference of the expected value of some $[0, 1]$-function $q_j$ under the distributions $\mathcal{S}(x_j)$ and $\mathcal{S}(x_j')$, i.e., $|\mathbb{E}_{\tilde{\chi} \sim \mathcal{S}(x_j)}[q_j(\tilde{\chi})] - \mathbb{E}_{\tilde{\chi} \sim \mathcal{S}(x_j')}[q_j(\tilde{\chi})]|$, which can be upper bounded by the total variation between $\mathcal{S}(x_j)$ and $\mathcal{S}(x_j')$.
Here, $q_j$ is given by:
\[q_j(\chi) = \mathbb{E}[f_i(\tilde{x}_i', \ldots, \tilde{x}_{j-1}', \chi, \tilde{x}_{j+1} \ldots, \tilde{x}_{i-s+1})],\]
where $\chi \in \mathcal{X}$ is the $j$th input item, the inputs before $\chi$ are drawn from the respective adversarially shifted smoothing distributions and the inputs after $\chi$ are drawn from the original distributions, i.e., $\tilde{x}_i' \sim \mathcal{S}(x_i'), \ldots, \tilde{x}_{j-1}' \sim \mathcal{S}(x_{j-1}')$ and $\tilde{x}_{j+1} \sim \mathcal{S}(x_{j+1}), \ldots, \tilde{x}_{i-s+1} \sim \mathcal{S}(x_{i-s+1})$.

Without loss of generality, assume $\mathbb{E}_{\tilde{\chi} \sim \mathcal{S}(x_j)}[q_j(\tilde{\chi})] \geq \mathbb{E}_{\tilde{\chi} \sim \mathcal{S}(x_j')}[q_j(\tilde{\chi})]$. Then,
\allowdisplaybreaks
\begin{align*}
    &\big| \mathbb{E}_{\tilde{\chi} \sim \mathcal{S}(x_j)}[q_j(\tilde{\chi})] - \mathbb{E}_{\tilde{\chi} \sim \mathcal{S}(x_j')}[q_j(\tilde{\chi})] \big|\\
    &= \int_{\mathcal{X}} q_j(x) \mu_1(x) dx - \int_{\mathcal{X}} q_j(x) \mu_2(x) dx \tag{$\mu_1$ and $\mu_2$ are the PDFs of $\mathcal{S}(x_j)$ and $\mathcal{S}(x_j')$}\\
    &= \int_{\mathcal{X}} q_j(x) (\mu_1(x) - \mu_2(x)) dx\\
    &= \int_{\mu_1 > \mu_2} q_j(x) (\mu_1(x) - \mu_2(x)) dx\\
    & \quad \quad - \int_{\mu_2 > \mu_1} q_j(x) (\mu_2(x) - \mu_1(x)) dx\\
    &\leq \int_{\mu_1 > \mu_2} \max_{x' \in \mathcal{X}}q_j(x') (\mu_1(x) - \mu_2(x)) dx\\
    & \quad \quad - \int_{\mu_2 > \mu_1} \min_{x' \in \mathcal{X}}q_j(x') (\mu_2(x) - \mu_1(x)) dx\\
    &\leq \int_{\mu_1 > \mu_2} (\mu_1(x) - \mu_2(x)) dz \tag{since $\max_{x' \in \mathcal{X}} q_j(x') \leq 1$ and $\min_{x' \in \mathcal{X}} q_j(x') \geq 0$}\\
    &= \frac{1}{2} \int_{\mathcal{X}} |\mu_1(x) - \mu_2(x)| dx =\TV(\mathcal{S}(x_1), \mathcal{S}(x_2)).\\
\end{align*}
The equality in the last line follows from the fact that $\int_{\mu_1 > \mu_2} (\mu_1(x) - \mu_2(x)) dx = \int_{\mu_2 > \mu_1} (\mu_2(x) - \mu_1(x)) dx = \frac{1}{2} \int_{\mathcal{X}} |\mu_1(x) - \mu_2(x)| dx$.

Therefore, from condition~(\ref{eq:tv_bound}), we have:
\begin{align*}
    |\tilde{f}_i(& x_i', \ldots x_j, \ldots, x_{i-w+1}) - \tilde{f}_i(x_i', \ldots, x_j', \ldots, x_{i-w+1})|\\
    &\leq \mathsf{TV}(\mathcal{S}(x_j), \mathcal{S}(x_j')) \leq \psi(d(x_j, x_j')).
\end{align*}
This proves the statement of the lemma.
\end{proof}


Now we use the above lemma to prove the main robustness guarantee.
We first decompose the change in the average performance into the average of the differences at each time step.
Then we apply lemma~\ref{lem:per-step-bound} to bound each difference with the function $\psi$ of the per-step perturbation size.
We then utilize the convex nature of $\psi$ to convert this average over the performance differences to an average of perturbation sizes, which completes the proof.
\begin{theorem}
\label{thm:performance-bound}
Let $\tilde{Z}_{\epsilon}$ to be the minimum $\tilde{Z}$ for an adversary in $\mathcal{A}_\epsilon$. Then,
\[|\tilde{Z} - \tilde{Z}_\epsilon | \leq w \psi(\epsilon).\]
\end{theorem}

\begin{proof}
Let $\tilde{Z}'$ be the overall performance of $\tilde{M}$ under an adversary. Then,
\allowdisplaybreaks
\begin{align*}
    |\tilde{Z} - \tilde{Z}'| &= \Bigg| \frac{\sum_{i=1}^t \tilde{f}_i(x_i, x_{i-1}, \ldots, x_{i-s+1})}{t}\\
    &- \frac{\sum_{i=1}^t \tilde{f}_i(x_i', x_{i-1}', \ldots, x_{i-s+1}')}{t} \Bigg| \tag{where $s = \min(i, w)$}\\
    &\leq \frac{1}{t} \sum_{i=1}^t \Big|\tilde{f}_i(x_i, x_{i-1}, \ldots, x_{i-s+1})\\
    &\quad \quad \quad \quad  - \tilde{f}_i(x_i', x_{i-1}', \ldots, x_{i-s+1}') \Big|\\
    &\leq \sum_{i=1}^t \sum_{j = i}^{i-s+1} \psi(d(x_j, x_j')) / t \tag{from lemma~\ref{lem:per-step-bound}}\\
    &\leq w \sum_{i=1}^t \psi(d(x_i, x_i')) / t \tag{since each term appears at most $w$ times}\\
    &\leq w \psi \left(\sum_{i=1}^t d(x_i, x_i') / t \tag{$\psi$ is concave and Jensen's inequality} \right)
\end{align*}
Therefore, for the worst-case adversary in $\mathcal{A}_\epsilon$, we have
\[|\tilde{Z} - \tilde{Z}_\epsilon | \leq w \psi(\epsilon)\]
from constraint~(\ref{eq:adv_constraint}) on the average distance between the original and perturbed inputs.
\end{proof}

Although the above certificate is designed for the sliding-window computational model for streaming applications, it may also be applied to the static tasks like classification with a fixed adversarial budget for all inputs by setting $w=1$.
In Appendix~\ref{sec:compare_static}, we compare the our bound with that obtained by \citet{cohen19} for an $\ell_2$-norm bounded adversary and a Gaussian smoothing distribution.
While the above bound is not tight, our analysis shows that the gap with static $\ell_2$-certificate is small for meaningful robustness guarantees.

\section{Attacking Each Window} \label{sec:attackingeahcwindow}
Now we consider the case where the adversary is allowed to attack each window seen by the target DNN separately.
The threat model in this section is defined using constraint~(\ref{eq:adv_constraint_window}).
It is able to re-attack an input item $x_i$ in each new window.
Similar to the definition of a window in Section~\ref{sec:notations}, define an adversarially corrupted window $W_i'$ as:
\begin{align*}
    W_i' = 
    \begin{cases}
    (x_1^i, x_2^{i-1}, \ldots, x_i^1) & \text{for } i \leq w\\
    (x_{i-w+1}^w, x_{i-w+2}^{w-1}, \ldots, x_i^1) & \text{otherwise,}
    \end{cases}
\end{align*}
where $x_i^k$ is the $k^\text{th}$ perturbed instance of $x_i$.

Similar to the certificate derived in Section~\ref{sec:rob_cert}, we first bound the change in the per-step performance function and then use that result to prove the final robustness guarantee.
We formulate the following lemma similar to Lemma~\ref{lem:per-step-bound} but accounting for the fact that each input item can be perturbed multiple times.


\begin{lemma}
\label{lem:per-window-bound}
The change in each $\tilde{f}_i$ under an adversary in $\mathcal{A}_\epsilon$ is bounded as
\[|\tilde{f}_i(W_i) - \tilde{f}_i(W_i')| \leq \sum_{j = i-s+1}^{i} \psi(d(x_j, x_j^{i+1-j})),\]
where $s = \min (i, w)$.
\end{lemma}

The proof is available in Appendix~\ref{proof:per-window-bound}.

We prove the same certified robustness bound as in Section~\ref{sec:rob_cert} but the $\epsilon$ here is defined according to constraint~(\ref{eq:adv_constraint_window}).
\begin{theorem}
\label{thm:performance-bound-each-window}
Let $\tilde{Z}_{\epsilon}$ to be the minimum $\tilde{Z}$ for an adversary in $\mathcal{A}_\epsilon$. Then,
\[|\tilde{Z} - \tilde{Z}_\epsilon | \leq w \psi(\epsilon).\]
\end{theorem}

The proof is available in Appendix~\ref{proof:performance-bound-each-window}.

\input{experiments.tex}

\section{Conclusion}
In this work, we design provable robustness guarantees for streaming machine learning models with a sliding window.
Our certificates provide a lower bound on the average performance of a streaming DNN model in the presence of an adversary.
The adversarial budget in our threat model is defined in terms of the average size of the perturbations added to the input items across the entire stream.
This allows the adversary to allocate a different budget to each input item and leads to a more general threat model than the static setting.
Our certificates are independent of the stream length and can handle long, potentially infinite, streams.
They are also applicable for adversaries that are allowed to re-attack past inputs leading to strong robustness guarantees covering a wide range of attack strategies.

Our robustness procedure simply augments the inputs with random noise.
Unlike conventional randomized smoothing techniques, our method only requires one noised sample per prediction keeping the computational requirements of the DNN model unchanged.
It does not make any assumptions about the DNN model such as Lipschitz continuity or a specific architecture and is applicable for conventional DNNs that are several layers deep.
Our experimental results show that our certificates can obtain meaningful robustness guarantees for real-world streaming applications.
Our results show that the certified performance of a robust model depends only on the window size and smaller windows lead to models that are provably more robust than larger windows.

To the best of our knowledge, this is the first attempt at designing adversarial robustness certificates for the streaming setting.
We note that our robustness guarantees are not proven to be tight and could be improved upon by future work.
We hope our work inspires further investigation into provable robustness methods for streaming ML models.

\section{Acknowledgements}
This project was supported in part by HR001119S0026 (GARD), NSF CAREER AWARD 1942230, ONR YIP award N00014-22-1-2271, Army Grant No. W911NF2120076, Meta grant 23010098, a capital one grant, NIST 60NANB20D134, and the NSF award CCF2212458.

\bibliography{references}
\bibliographystyle{icml2023}

\newpage
\appendix
\onecolumn

\section{Proof of Lemma~\ref{lem:per-window-bound}}
\label{proof:per-window-bound}
\begin{statement}
The change in each $\tilde{f}_j$ under an adversary in $\mathcal{A}_\epsilon$ is bounded as
\[|\tilde{f}_j(W_j) - \tilde{f}_j(W_j')| \leq \sum_{i = j-w+1}^{j} \psi(d(x_i, x_i^{j+1-i})).\]
\end{statement}

\begin{proof}
The left-hand side of the above inequality can be re-written as:
\begin{align*}
    |\tilde{f}_j(W_j) - \tilde{f}_j(W_j')| &= |\tilde{f}_j(x_{j-w+1}, \ldots, x_j) - \tilde{f}_j(x_{j-w+1}^w, \ldots, x_j^1)|\\
    &= |\tilde{f}_j(x_{j-w+1}, \ldots, x_{j-1}, x_j) - \tilde{f}_j(x_{j-w+1}, \ldots, x_{j-1}, x_j^1)\\
    &+\tilde{f}_j(x_{j-w+1}, \ldots, x_{j-1}, x_j^1) - \tilde{f}_j(x_{j-w+1}^w, \ldots, x_{j-1}^2, x_j^1)|\\
    &= \Bigg| \sum_{k = 1}^w \tilde{f}_j(x_{j-w+1}, \ldots x_{j-k+1}, x_{j-k+2}^{k-1}, \ldots, x_j^1) - \tilde{f}_j(x_{j-w+1}, \ldots x_{j-k+1}^k, x_{j-k+2}^{k-1}, \ldots, x_j^1) \Bigg|\\
    &\leq \sum_{k = 1}^w \left|\tilde{f}_j(x_{j-w+1}, \ldots x_{j-k+1}, x_{j-k+2}^{k-1}, \ldots, x_j^1) - \tilde{f}_j(x_{j-w+1}, \ldots x_{j-k+1}^k, x_{j-k+2}^{k-1}, \ldots, x_j^1)\right|
\end{align*}
The two terms in each summand differ only in the $(j-k+1)$-th input.
Thus, it can be written as the difference of the expected value of some $[0, 1]$-function $q$ under the distributions $\mathcal{S}(x_{j-k+1})$ and $\mathcal{S}(x_{j-k+1}^k)$, i.e., $|\mathbb{E}_{\tilde{x}_{j-k+1} \sim \mathcal{S}(x_{j-k+1})}[q(\tilde{x}_{j-k+1})] - \mathbb{E}_{\tilde{x}_{j-k+1}^k \sim \mathcal{S}(x_{j-k+1}^k)}[q(\tilde{x}_{j-k+1}^k)]|$
which can be upper bounded by the total variation between $\mathcal{S}(x_{j-k+1})$ and $\mathcal{S}(x_{j-k+1}^k)$.
Therefore, from condition~(\ref{eq:tv_bound}), we have:
\begin{align*}
    |\tilde{f}_j(x_{j-w+1}, \ldots, &x_{j-k+1}, x_{j-k+2}^{k-1}, \ldots, x_j^1) - \tilde{f}_j(x_{j-w+1}, \ldots x_{j-k+1}^k, x_{j-k+2}^{k-1}, \ldots, x_j^1)|\\
    &\leq \mathsf{TV}(\mathcal{S}(x_{j-k+1}), \mathcal{S}(x_{j-k+1}^k)) \leq \psi(d(x_{j-k+1}, x_{j-k+1}^k)).
\end{align*}
This proves the statement of the lemma.
\end{proof}

\section{Proof of Theorem~\ref{thm:performance-bound-each-window}}
\label{proof:performance-bound-each-window}

\begin{statement}
    Let $\tilde{Z}_{\epsilon}$ to be the minimum $\tilde{Z}$ for an adversary in $\mathcal{A}_\epsilon$. Then,
\[|\tilde{Z} - \tilde{Z}_\epsilon | \leq w \psi(\epsilon).\]
\end{statement}

\begin{proof}
Let $\tilde{Z}'$ be the overall performance of $\tilde{M}$ under an adversary. Then,
\begin{align*}
    |\tilde{Z} - \tilde{Z}'| &= \left| \frac{\sum_{j=1}^t \tilde{f}_j(W_j)}{t} - \frac{\sum_{j=1}^t \tilde{f}_j(W_j')}{t} \right|\\
    &\leq \frac{\sum_{j=1}^t |\tilde{f}_j(W_j)) - \tilde{f}_j(W_j')|}{t}\\
    &\leq \sum_{j=1}^t \sum_{k = 1}^w \psi(d(x_{j-k+1}, x_{j-k+1}^k)) / t \tag{from lemma~\ref{lem:per-window-bound}}\\
    &\leq \sum_{j=1}^t \sum_{k = 1}^w \psi(d(x_j, x_j^k)) / t\\
    &= w \sum_{j=1}^t \sum_{k = 1}^w \psi(d(x_j, x_j^k)) / w t\\
    &\leq w \psi \left( \sum_{j=1}^t \sum_{k = 1}^w d(x_j, x_j^k)) / w t \right) \tag{$\psi$ is concave and Jensen's inequality}
\end{align*}
Therefore, for the worst-case adversary in $\mathcal{A}_\epsilon$, we have
\[|\tilde{Z} - \tilde{Z}_\epsilon | \leq w \psi(\epsilon)\]
from constraint~(\ref{eq:adv_constraint}) on the average distance between the original and perturbed inputs.
\end{proof}

\section{Function $\psi$ for Different Distributions}
\label{sec:psi-functions}
For an isometric Gaussian distribution,
\[\TV (\mathcal{N}(x_i, \sigma^2 I), \mathcal{N}(x_i', \sigma^2 I)) = \erf (\|x_i - x_i'\|_2/2 \sqrt{2} \sigma).\]

\begin{proof}
Due to the isometric symmetry of the Gaussian distribution and the $\ell_2$-norm, the total variation between the two distributions is the same as when they are separated by the same $\ell_2$-distance but only in the first coordinate.
It is equivalent to shifting a univariate normal distribution by the same amount.
Therefore, the total variation between the two distributions is equal to the difference in the probability of a normal random variable with variance $\sigma^2$ being less than $\|x_i - x_i'\|_2/2$ and $-\|x_i - x_i'\|_2/2$, i.e., $\Phi(\|x_i - x_i'\|_2/2 \sigma) - \Phi(-\|x_i - x_i'\|_2/2\sigma)$ where $\Phi$ is the standard normal CDF.
\begin{align*}
    \TV (\mathcal{N}(x_i, \sigma^2 I), \mathcal{N}(x_i', \sigma^2 I)) &= \Phi(\|x_i - x_i'\|_2/2 \sigma) - \Phi(-\|x_i - x_i'\|_2/2 \sigma)\\
    &= 2\Phi(\|x_i - x_i'\|_2/2 \sigma) - 1\\
    &= 2 \left( \frac{1 + \erf( \|x_i - x_i'\|_2 / 2 \sqrt{2} \sigma)}{2} \right) - 1\\
    &= \erf( \|x_i - x_i'\|_2 / 2 \sqrt{2} \sigma).
\end{align*}
\end{proof}

For a uniform smoothing distribution $\mathcal{U}(x_i, b)$ between ${x_i}_j-b/2$ and ${x_i}_j + b/2$ in each dimension $j$ of $x_i$ for some $b \geq 0$, $\TV (\mathcal{U}(x_i, b), \mathcal{U}(x_i', b)) \leq \|x_i - x_i'\|_1/b$.
When $\|x_i - x_i'\|_1$ is constrained, the overlap between $\mathcal{U}(x_i, b)$ and $\mathcal{U}(x_i', b)$ is minimized when the shift is only along one dimension.

\section{Comparison with Existing Certificates for Static Tasks}
\label{sec:compare_static}
\begin{wrapfigure}{r}{0.5\textwidth}
    \vspace{-6mm}
    \hspace{-8mm}
    \includegraphics[width=0.6\textwidth,trim={0 0 0 1cm},clip]{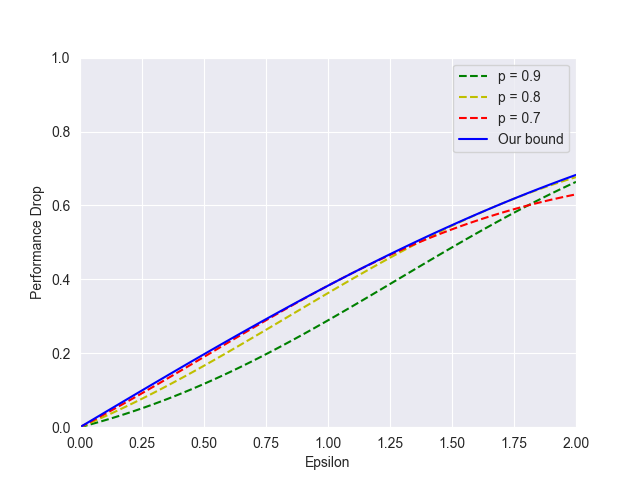}
    \vspace{-6mm}
    \caption{Comparison between our bound and \citet{cohen19}'s certificate for an $\ell_2$ adversary and a Gaussian smoothing distribution. The solid blue curve corresponds to our bound and the dashed curves represent bound~(\ref{eq:cohen_bound}) for different values of $p$.
    We keep $\sigma=1$ as it only has a scaling effect along the $x$-axis.}
    \vspace{-4mm}
    \label{fig:bound_compare}
\end{wrapfigure}
In this section, we compare our bound when applied to the static setting of classification, i.e., window size $w = 1$ in bound~(\ref{eq:main_result}), to that obtained by \citet{cohen19} for an $\ell_2$ adversary and a Gaussian smoothing distribution.
As discussed in Appendix~\ref{sec:psi-functions}, the $\psi$ function for this case takes the form of the Gauss error function $\erf$.
Thus our bound on the drop in the smoothed model's performance against an $\ell_2$ adversary is given by:
\[|\tilde{Z} - \tilde{Z}_\epsilon | \leq \erf(\epsilon/2\sqrt{2}\sigma).\]
\citet{cohen19}'s certificate bounds the worst-case adversarial performance as a function of the clean performance.
If the probability of predicting the correct class is $p$ on the original input, the probability of that in the presence of an adversary is bounded by $\Phi (\Phi^{-1}(p) - \epsilon/\sigma)$.
Therefore, the performance drop $\Delta p$ is bounded by:
\begin{equation}
\label{eq:cohen_bound}    
\Delta p \leq p - \Phi \left(\Phi^{-1}(p) - \frac{\epsilon}{\sigma} \right).
\end{equation}

Figure~\ref{fig:bound_compare} compares the two bounds for different values of $p$.
We keep $\sigma=1$ as it only has a scaling effect along the $x$-axis.
The bound from the $\ell_2$ certificate by \citet{cohen19} is tighter than ours, mainly because it takes the clean performance $p$ of the smoothed model into account.
However, the gap between the two bounds is small in the range where $\epsilon$ goes from 0 to 2, by which point the certified performance drops by more than 60\%.
Thus for most meaningful robustness guarantees, our certificates are almost at par with the best-known $\ell_2$ certificates.
The key advantage of our certificates over those for the static setting is that they are applicable for an adaptive adversary that can allocate different attack budgets for different input items in the stream.

\section{Experimental details}
\label{app:exps}

We use a single NVIDIA RTX A4000 GPU with four AMD EPYC 7302P Processors. For our main experiments with UCI HAR and Speech Commands datasets, we use window size $w=2$ with inputs belonging to $\mathbb{R}^{250\times 6}$ and $\mathbb{R}^{4000}$. The UCI HAR dataset consists of long streaming inputs with sample-level annotations. For a window $W_j$, the label is the majority class that is present in that window. The signals in the HAR dataset are standardized to have mean 0 and variance 1. For the speech keyword detection task, we use a subset of the Speech commands dataset that consists of long noise clips and one-second-long speech keyword clips. The labels for each audio clip are available. We utilize all the long noise clips and clips belonging to the classes belonging speech utterances of numbers from zero to nine to make longer clips for our streaming case. We add noise clips to the keyword audios to make them more similar to real-world scenarios. Each clip is stitched together \citep{audiostitch} with arbitrarily long noise between each keyword clip. To make transitions between the audio smooth, we use exponential decays to overlap keyworrd audio clips for stitching, with noise in the background. Hence, for the speech keyword detection, we have 11 classes for labels -- zero to nine and a noise class. A window is labeled to be the majority class in that window.

For training, we use M5 networks with 32 channels for HAR. We train for 30 epochs with a bath-size of 256 using SGD with an initial learning rate of 0.1, momentum of 0.9, and weight decay of 0.0001. We use a cosine annealing learning rate scheduler. For training the robust models, we use different smoothing noises with standard deviations 4, 6, 8, and 10. For training on the keyword detection data, we use M5 networks with 128 channels for HAR. We train for 30 epochs with a bath-size of 128 using SGD with an initial learning rate of 0.1, momentum of 0.9, and weight decay of 0.0001. We use a cosine annealing learning rate scheduler. For training the robust models, we use different smoothing noises with standard deviations 0.1, 0.2, 0.4, 0.6, and 0.8. For attacking the trained models, we use PGD $\ell_2$ attacks for both the datasets. PGD is run for 100 steps with a step size of $2 \epsilon'/100$ where $\epsilon'$ is the $\ell_2$ attack budget. 

\section{Attacking the Smooth Models}
\label{app:attacksmooth}

In this section, we empirically validate our certificates by showing that the performance of the smoothed models in the presence of an adversary is lower-bounded by our certificates.
For the first set of experiments (Figures~\ref{fig:attacksmoothhar} and ~\ref{fig:attacksmoothspeech}), we consider an adversary that is allowed to attack an input item only once, as in Section \ref{sec:attackonce}. We show our results on the Human Activity Recognition dataset in Figure \ref{fig:attacksmoothhar} and the keyword detection task in Figure \ref{fig:attacksmoothspeech} for a window size of $w=2$. In Figure \ref{fig:attacksmoothhar_2}, we show our results on the HAR dataset where the adversary can attack each window separately as per equation \ref{eq:adv_constraint_window}. As seen in the plots, the empirical performance of the smooth models after the online adversarial attacks is always better than the performance guaranteed by our certificates.
By comparing Figures~\ref{fig:attacksmoothhar} and~\ref{fig:attacksmoothhar_2}, we observe that allowing the adversary to attack each window separately makes it significantly stronger and brings the adversarial performance of the smoothed model closer to the certified performance.

\begin{figure}[t]
     \centering
    \subfigure[Attacking model with smoothing noise $\sigma=4$]{\includegraphics[width=0.45\linewidth]{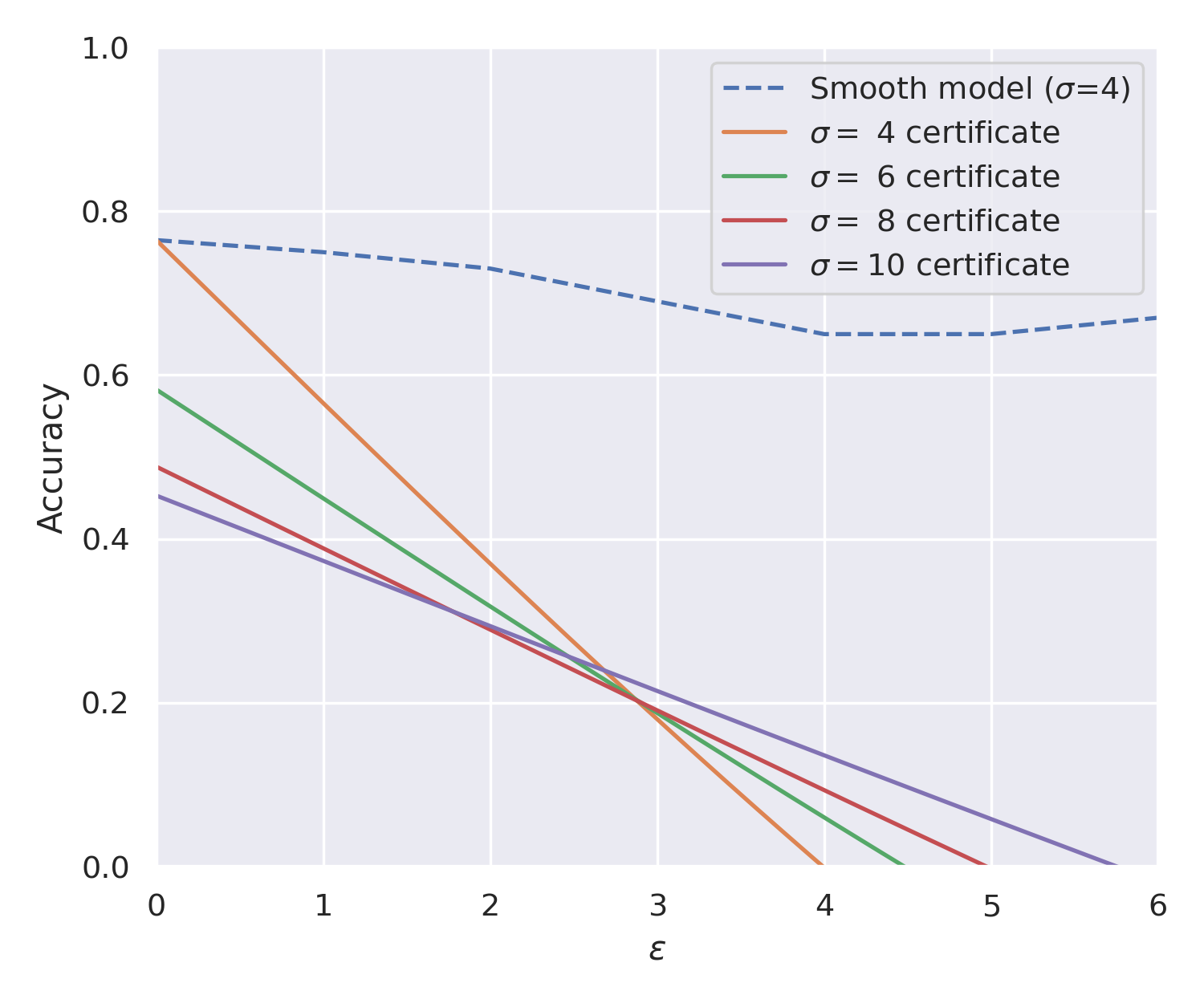}\label{fig:attacksmoothhar1}}%
    \subfigure[Attacking model with smoothing noise $\sigma=6$]{\includegraphics[width=0.45\linewidth]{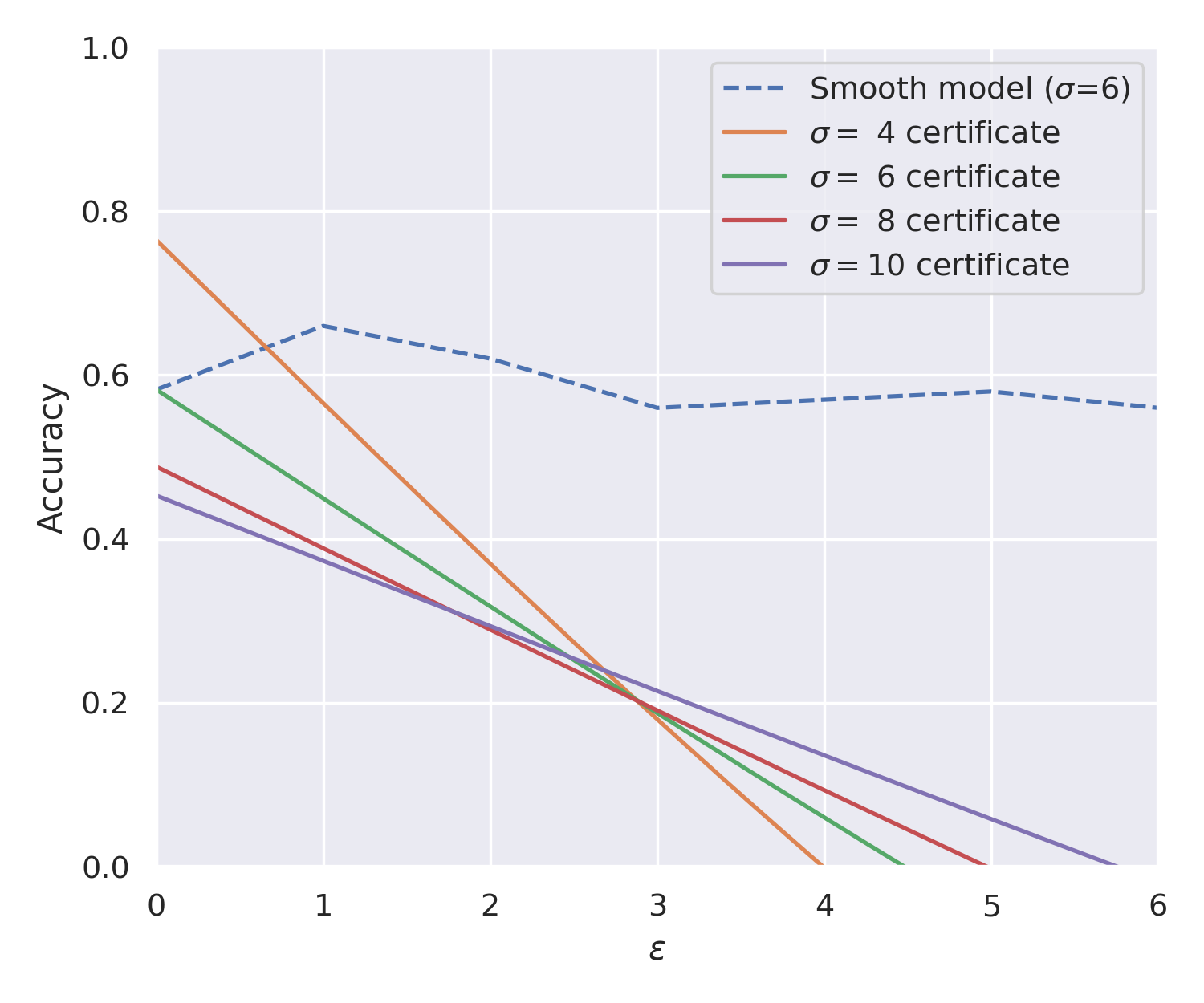}\label{fig:attacksmoothhar2}}\\
    \subfigure[Attacking model with smoothing noise $\sigma=8$]{\includegraphics[width=0.45\linewidth]{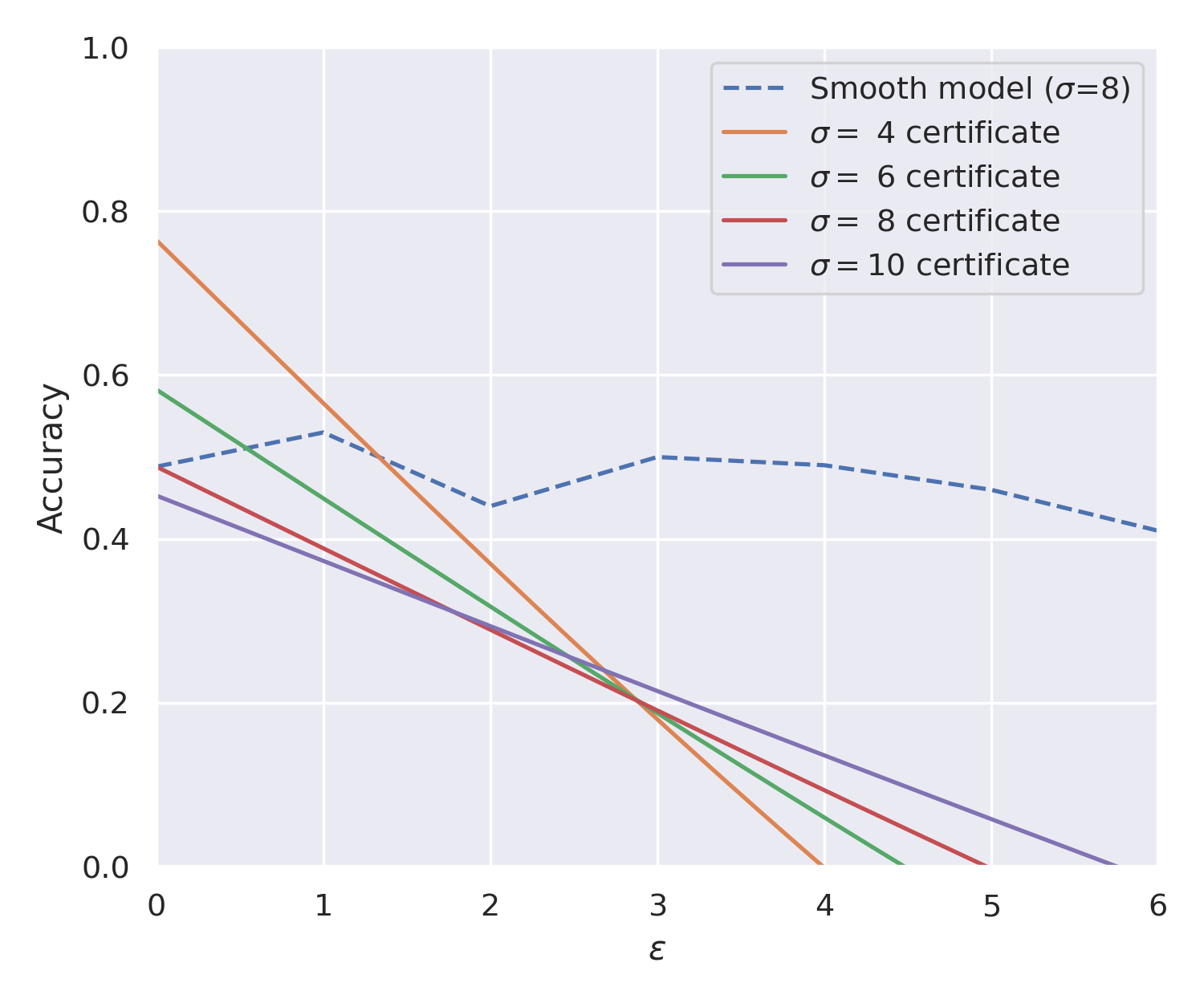}\label{fig:attacksmoothhar3}}%
    \subfigure[Attacking model with smoothing noise $\sigma=10$]{\includegraphics[width=0.45\linewidth]{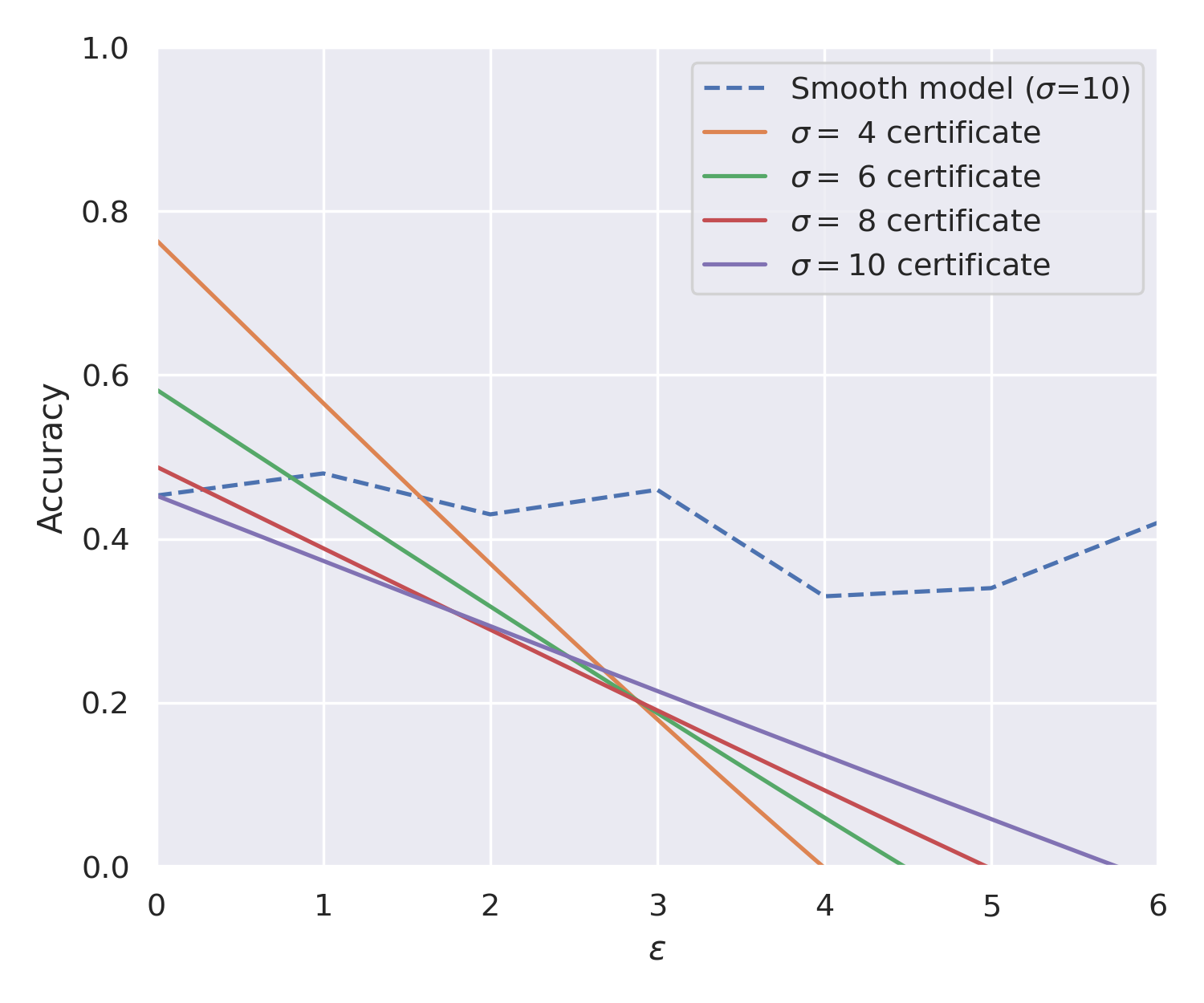}\label{fig:attacksmoothhar4}}
        \caption{Certificates against online adversarial attacks for varying smoothing noises for the human activity recognition task. We attack smooth models trained with different smoothing noises in these plots. Here we can perturb each input only once. The average size of perturbation is computed as per equation \ref{eq:adv_constraint}.}
        \label{fig:attacksmoothhar}
        \vspace{-3mm}
\end{figure}

\begin{figure}[t]
     \centering
    \subfigure[Attacking model with smoothing noise $\sigma=0.2$]{\includegraphics[width=0.45\linewidth]{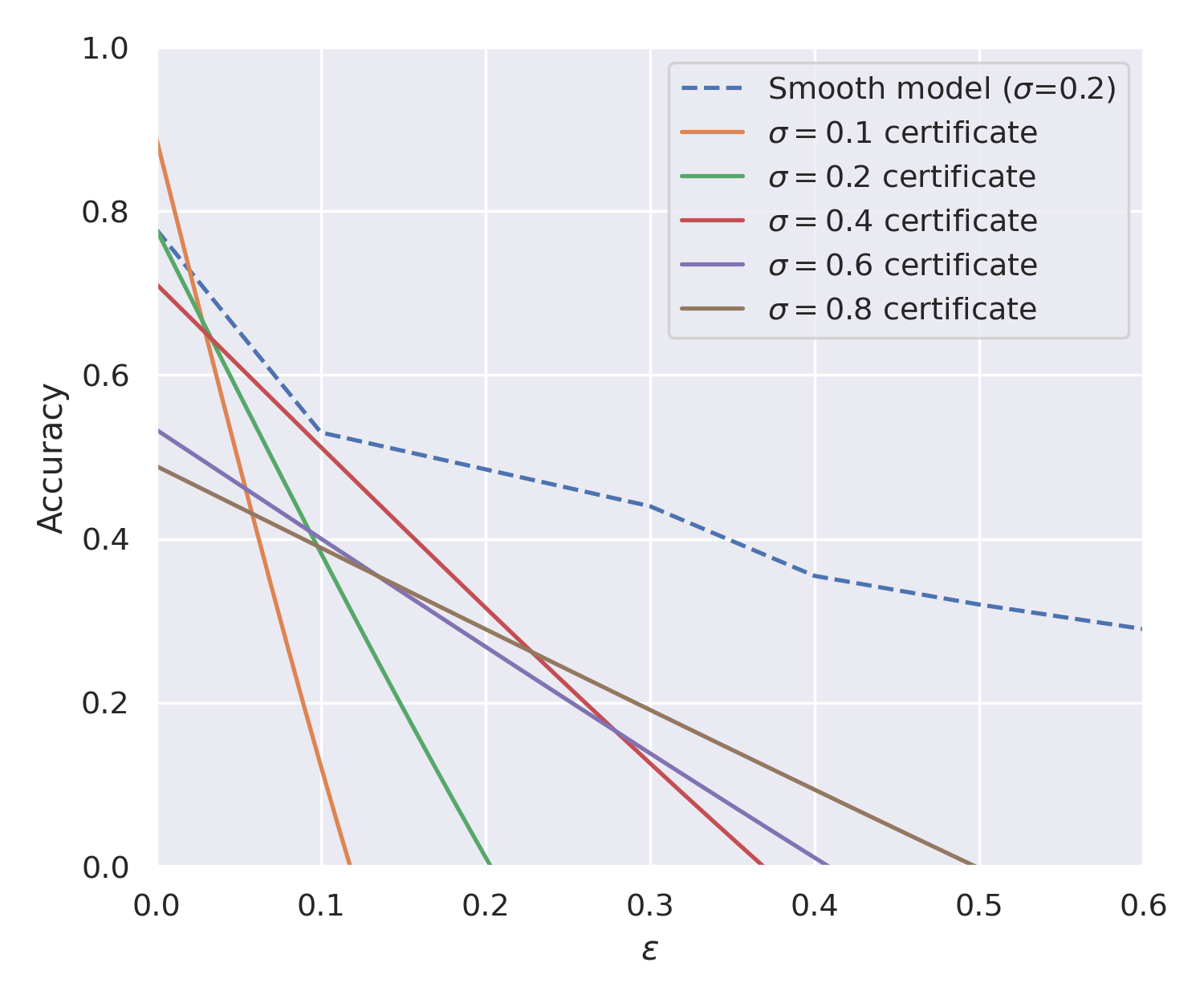}\label{fig:attacksmoothspeech2}}%
    \subfigure[Attacking model with smoothing noise $\sigma=0.4$]{\includegraphics[width=0.45\linewidth]{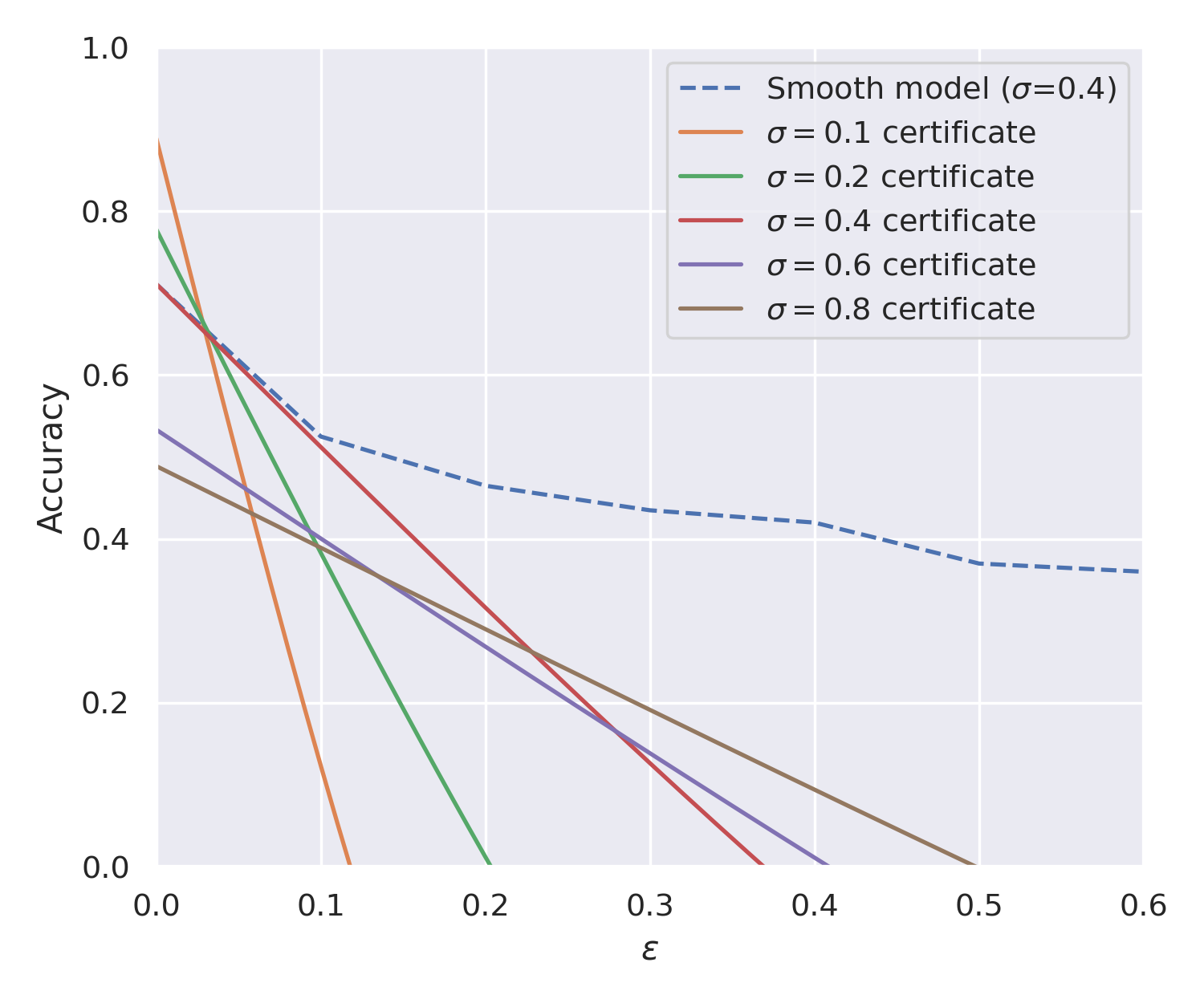}\label{fig:attacksmoothspeech6}}\\
    \subfigure[Attacking model with smoothing noise $\sigma=0.6$]{\includegraphics[width=0.45\linewidth]{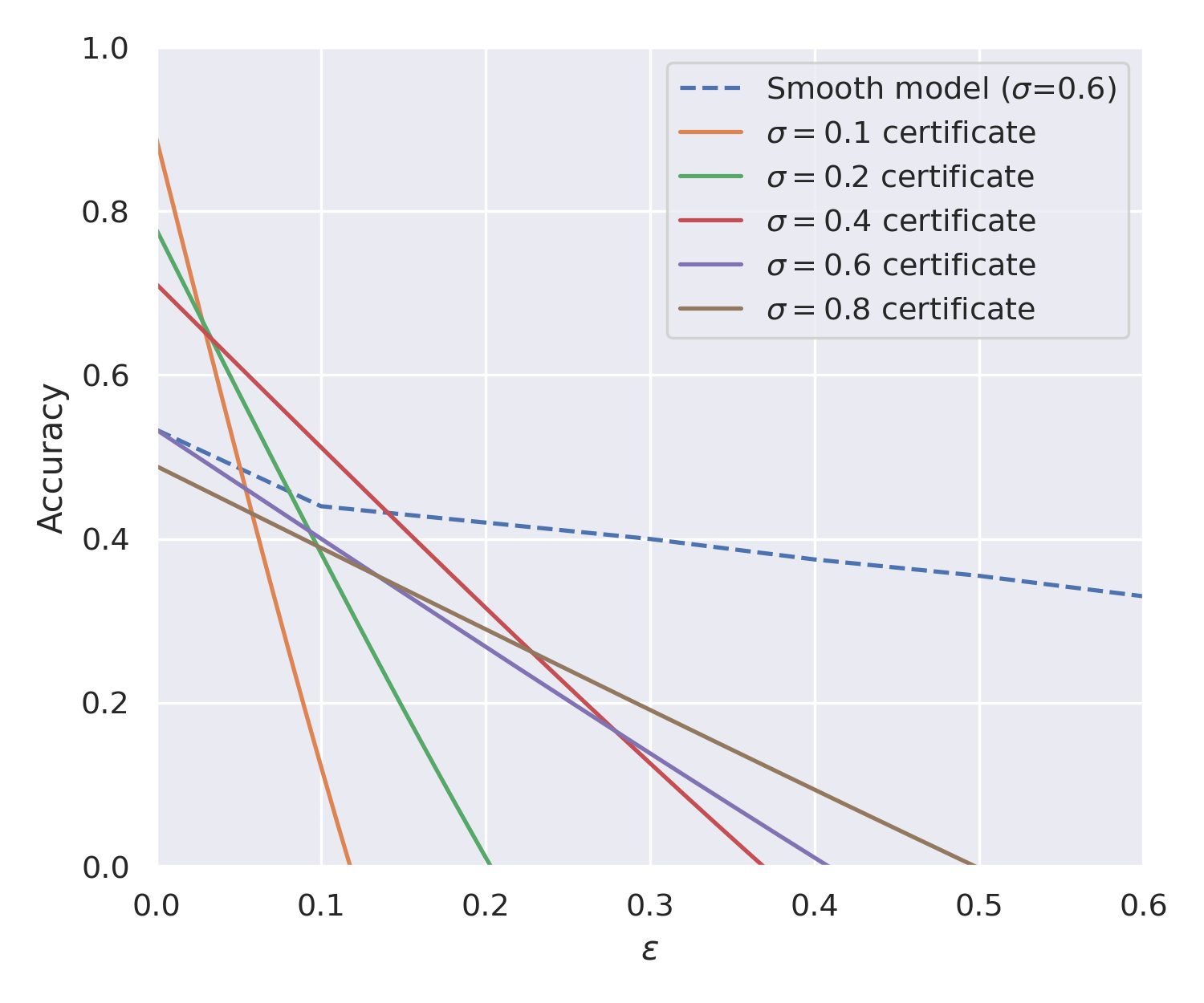}\label{fig:attacksmoothspeech8}}%
    \subfigure[Attacking model with smoothing noise $\sigma=0.8$]{\includegraphics[width=0.45\linewidth]{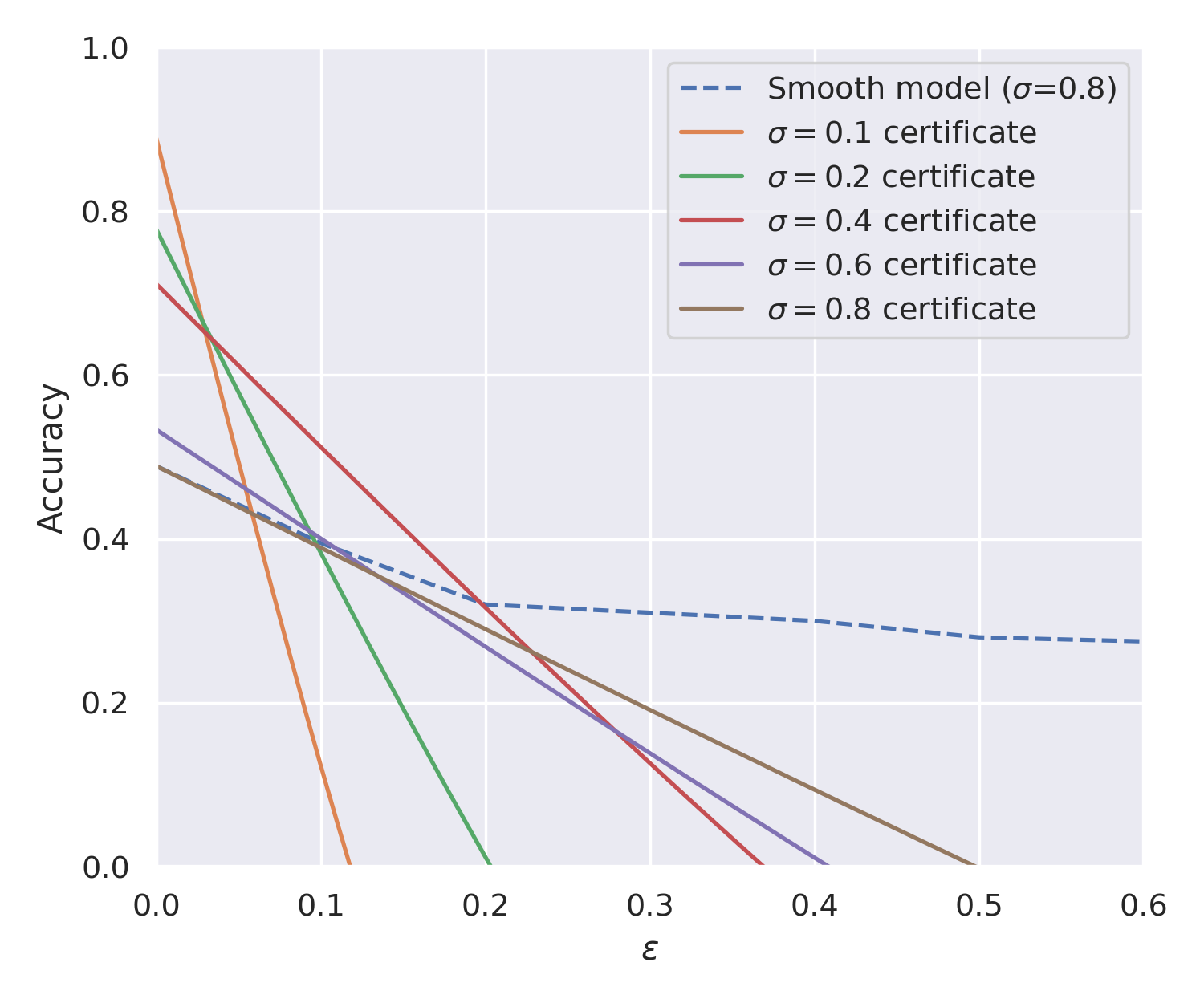}\label{fig:attacksmoothspeech10}}
        \caption{Certificates against online adversarial attacks for varying smoothing noises for the speech keyword detection task. We attack smooth models trained with different smoothing noises in these plots. Here we can perturb each input only once. The average size of perturbation is computed as per equation \ref{eq:adv_constraint}.}
        \label{fig:attacksmoothspeech}
        \vspace{-3mm}
\end{figure}

\begin{figure}[t]
     \centering
    \subfigure[Attacking model with smoothing noise $\sigma=4$]{\includegraphics[width=0.45\linewidth]{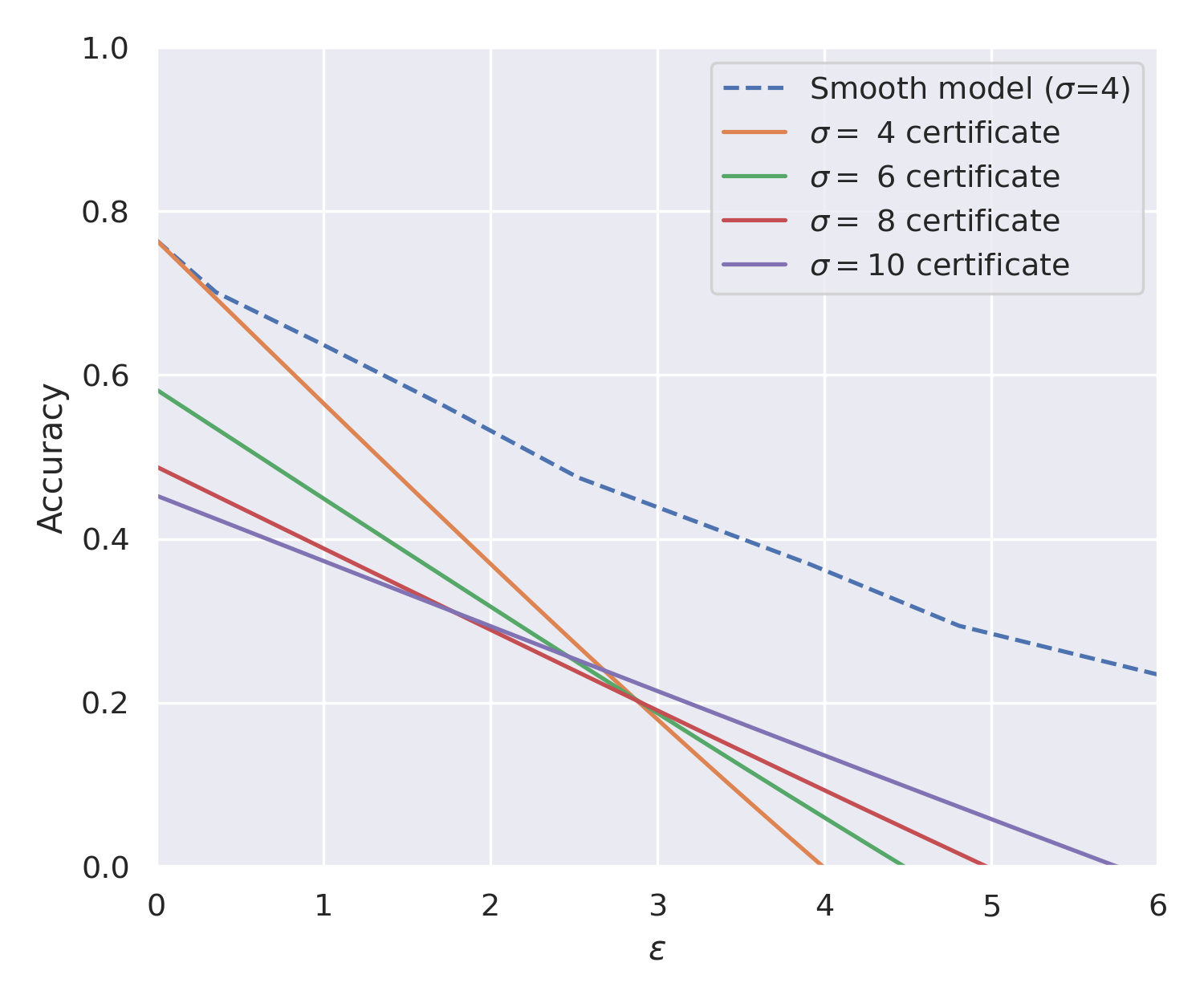}\label{fig:attacksmoothhar_2_4}}%
    \subfigure[Attacking model with smoothing noise $\sigma=6$]{\includegraphics[width=0.45\linewidth]{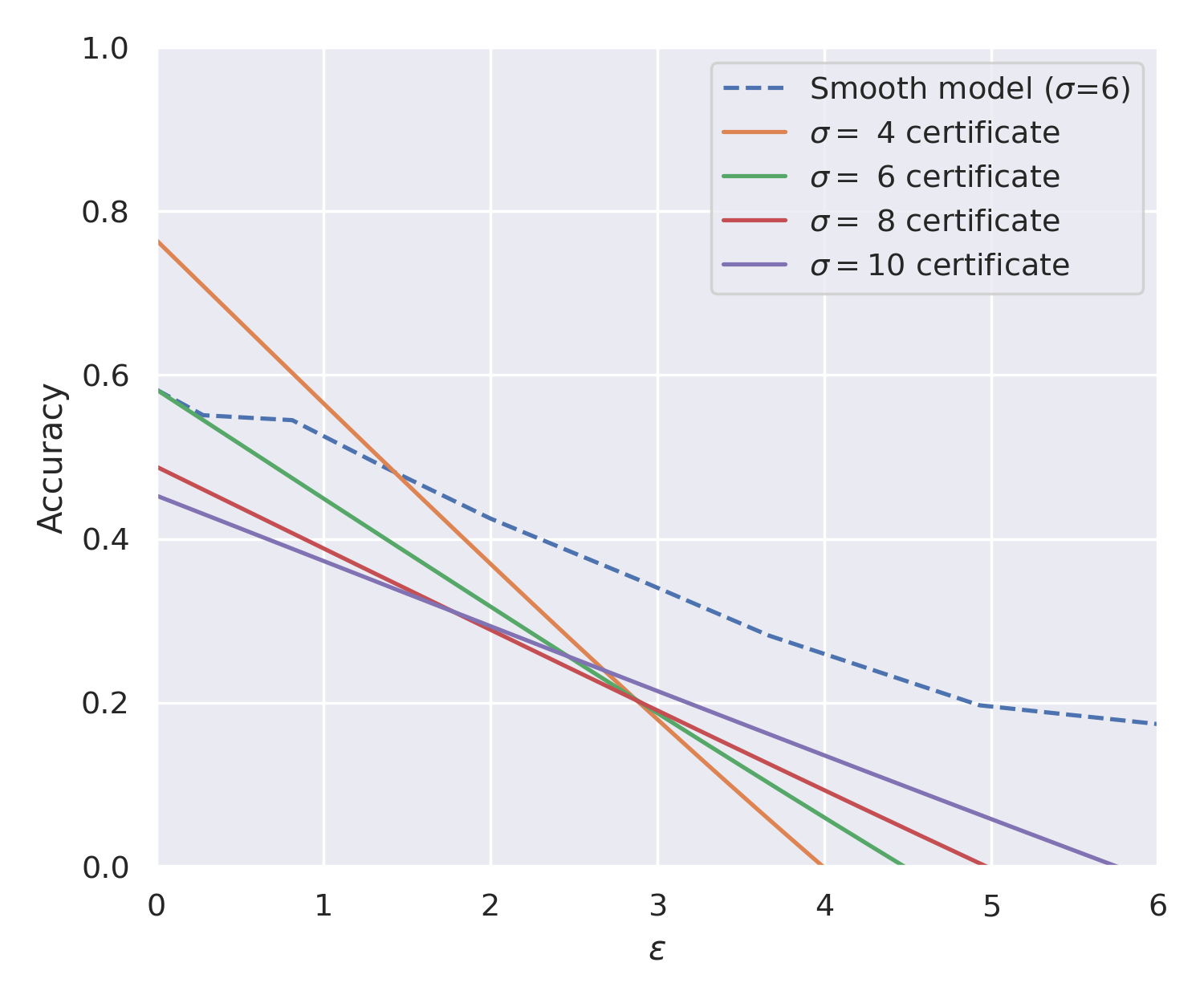}\label{fig:attacksmoothhar_2_6}}\\
    \subfigure[Attacking model with smoothing noise $\sigma=8$]{\includegraphics[width=0.45\linewidth]{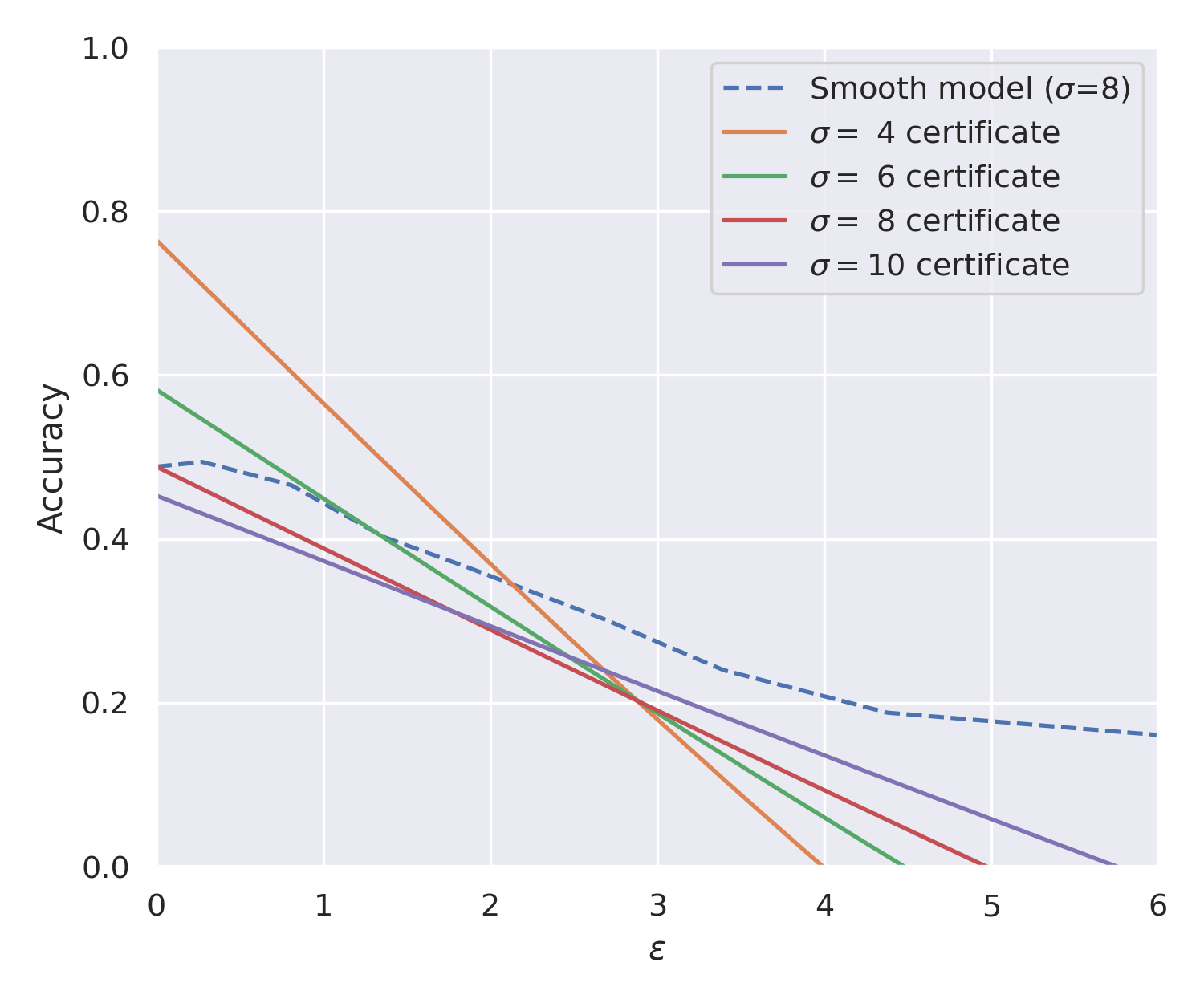}\label{fig:attacksmoothhar_2_8}}%
    \subfigure[Attacking model with smoothing noise $\sigma=10$]{\includegraphics[width=0.45\linewidth]{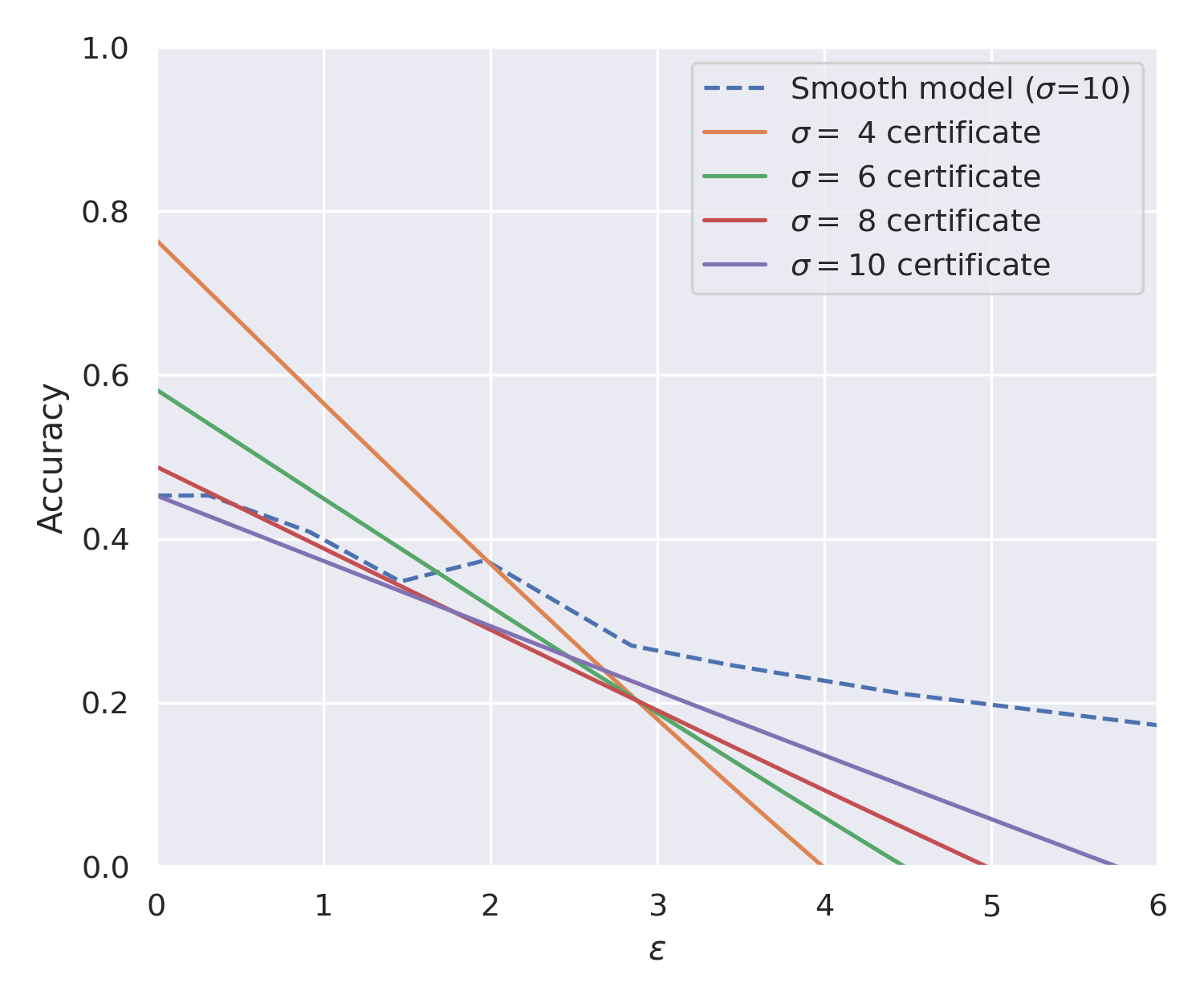}\label{fig:attacksmoothhar_2_10}}
        \caption{Certificates against online adversarial attacks for varying smoothing noises for the human activity recognition task. We attack smooth models trained with different smoothing noises in these plots. Here we can attack each window separately. The average size of perturbation is computed as per equation \ref{eq:adv_constraint_window}.}
        \label{fig:attacksmoothhar_2}
        \vspace{-3mm}
\end{figure}


\end{document}

%% file: experiments.tex
\section{Experiments}

\begin{figure}[t]
     \centering
    \subfigure[Speech keyword detection]{\includegraphics[width=0.9\linewidth]{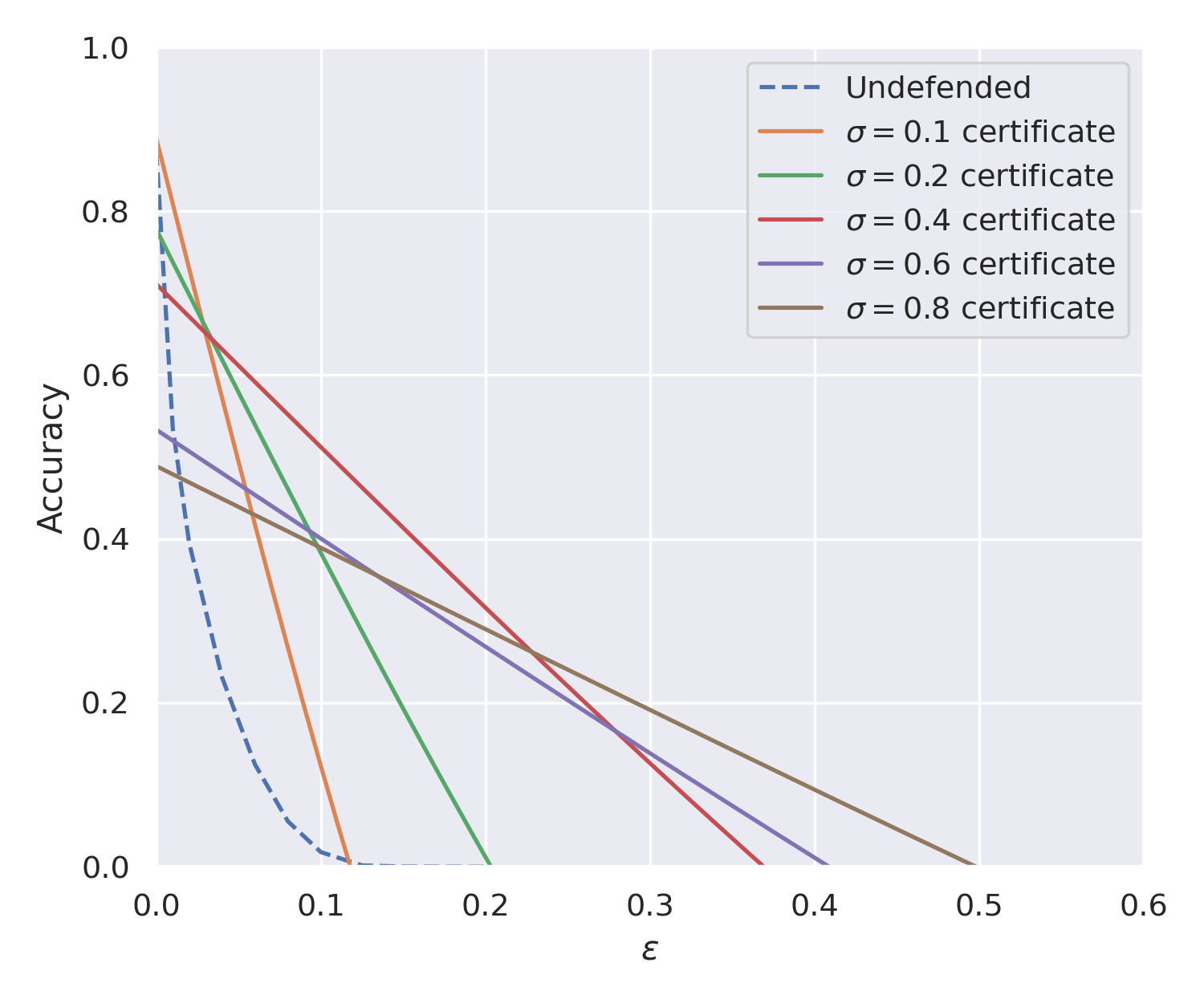}\label{fig:speech1}}\\
    \subfigure[Human activity recognition]{\includegraphics[width=0.9\linewidth]{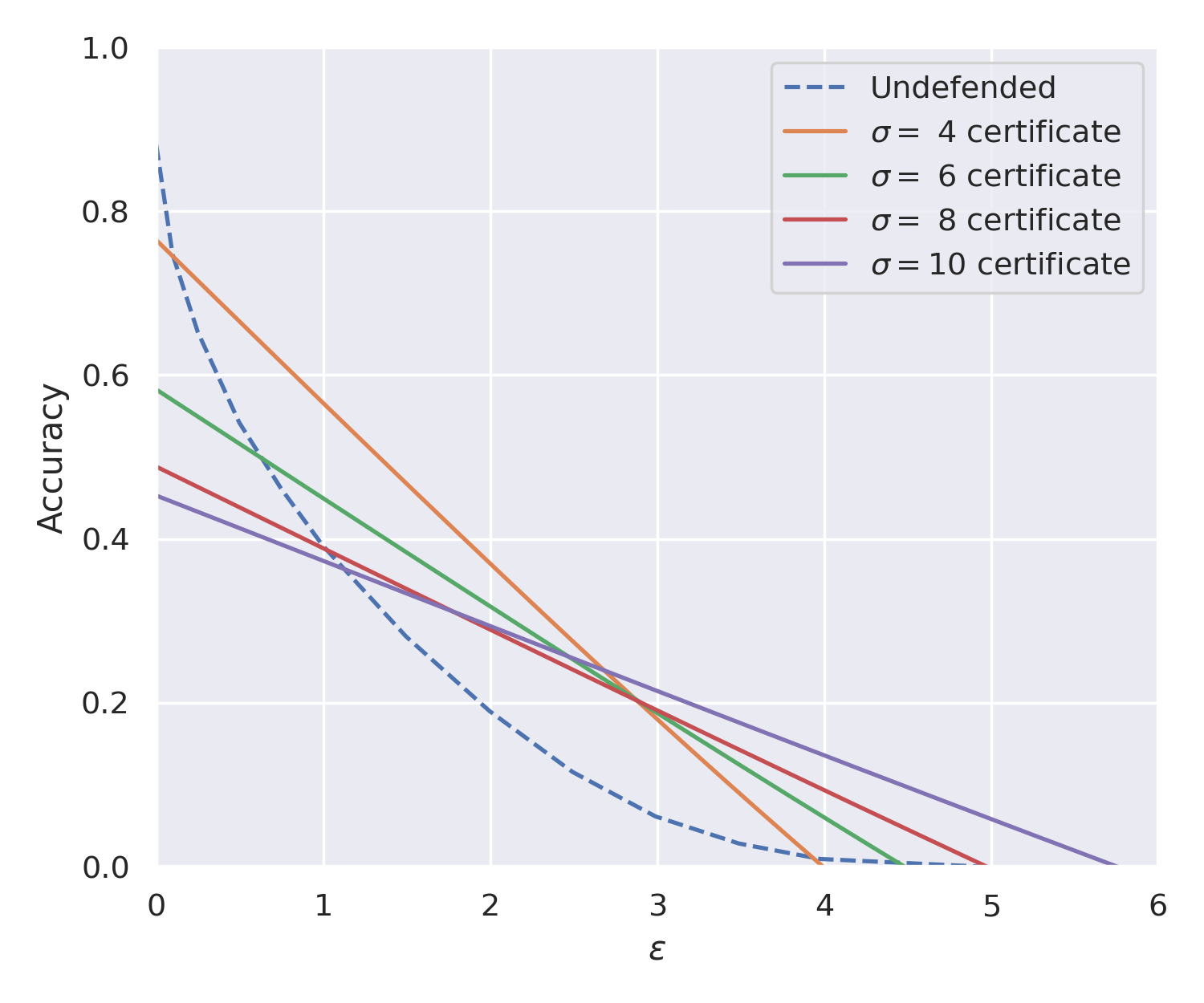}\label{fig:har1}}
        \caption{Certificates against online adversarial attacks for varying smoothing noises. Here we can perturb each input only once. The average size of perturbation is computed as per equation \ref{eq:adv_constraint}.}
        \label{fig:certificates1}
        \vspace{-3mm}
\end{figure}

We test our certificates for two streaming tasks -- speech keyword detection and human activity recognition. We use a subset of the Speech commands dataset \citep{speechcommandsv2} for our speech keyword detection task. The subset we use contains ten keyword classes, corresponding to utterances of numbers from zero to nine recorded at a sample rate of 16 kHz. This dataset also contains noise clips such as audio of running tap water and exercise bike. We add these noise clips to the speech audio to simulate real-world scenarios and stitch them together to generate longer audio clips. We use the UCI HAR dataset \citep{reyes2012uci} for human activity recognition. This contains a 6-D triaxial accelerometer and gyroscope readings measured with human subjects. The objective in HAR is to recognize various human activities based on sensor readings. The UCI HAR dataset contains signals recorded at 50 Hz that correspond to six human activities such as standing, sitting, laying, walking, walking up, and walking down.

We use the M5 network described in \cite{m5} with an SGD optimizer and an initial learning rate of 0.1, which we anneal using a cosine scheduler. For the speech detection task, we train a M5 network with 128 channels for 30 epochs with a batch size of 128. For the human activity recognition task, we use a M5 network with 32 channels for 30 epochs with a batch size of 256. We apply isotropic Gaussian noise for smoothing and use the $\ell_2$-norm to define the average distance measure $d$. For the speech keyword detection task, we use smoothing noises with standard deviations of 0.1, 0.2, 0.4, 0.6, and 0.8. For the human activity recognition task, we use smoothing noises with standard deviations of 4, 6, 8, and 10. See Appendix \ref{app:exps} for more details on the experiments. We compute certificates for both scenarios, where the input is attacked only once and where each window can be attacked with the ability to re-attack inputs. These experiments show that our certificates provide 
 meaningful guarantees against adversarial perturbations.

\begin{algorithm}[tb]
   \caption{Our streaming attack}
   \label{alg:onlineattack}
\begin{algorithmic}
   \STATE {\bfseries Input:} time-step $j$, clean inputs $x_j, x_{j-1}, ..., x_{j-w+1}$, perturbed inputs $x'_{j-1}, ..., x'_{j-w+1}$, attack budget $\epsilon$, search parameter $\alpha \in \mathbb{N}$.
   \STATE $d_{j-1} = \sum_{i=1}^{j-1} d(x_i, x_i')$
   \STATE $budget_j = j \epsilon - d_{j-1}$
   \FOR{$i=0$ {\bfseries to} $\alpha$}
   \STATE $\epsilon' = \frac{i}{\alpha} \cdot budget_j$
   \STATE $x = \arg \min_{x} f_j(x,...,x_{j-w+1}')$ s.t. $d(x, x_j) \leq \epsilon'$
   \IF{$f_j(x_j', ...,x_{j-w+1}') = 0$}
   \STATE $x_j' = x$
   \STATE break
   \ELSE 
   \STATE $x_j' = x_j$
   \ENDIF
   \ENDFOR
\end{algorithmic}
\end{algorithm}

\subsection{Attacking an input only once} \label{sec:attackonce}

\begin{figure}[t]
     \centering
    \subfigure[Speech keyword detection]{\includegraphics[width=0.9\linewidth]{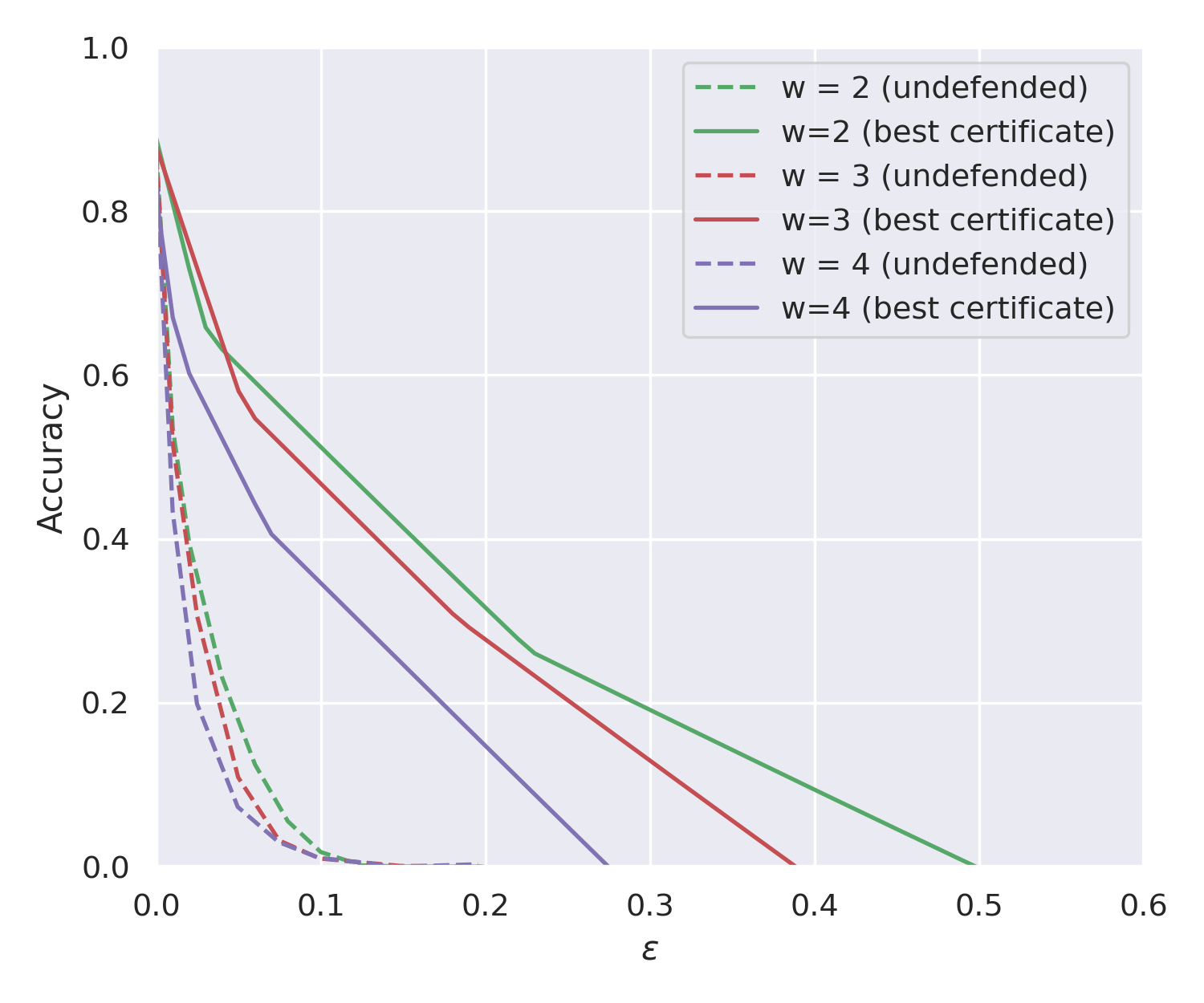}\label{fig:bestspeech}}\\
    \subfigure[Human activity recognition]{\includegraphics[width=0.9\linewidth]{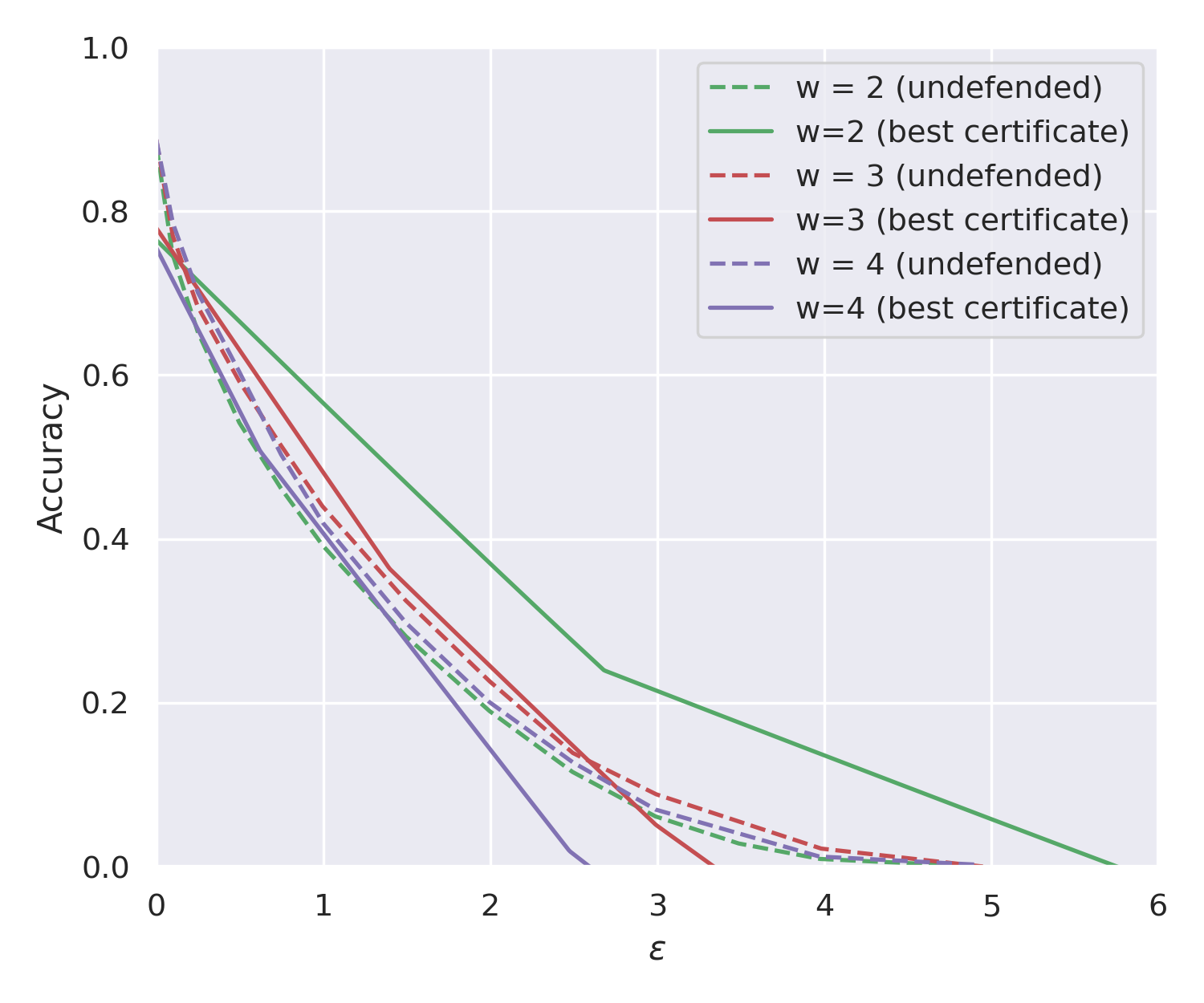}\label{fig:besthar}}
        \caption{Best certificates across varying smoothing noises for different window sizes. Streaming models with smaller window sizes are more robust to adversarial perturbations.}
        \label{fig:bestcertificates1}
        \vspace{-6mm}
\end{figure}

\begin{figure}[t]
     \centering
    \subfigure[Speech keyword detection]{\includegraphics[width=0.9\linewidth]{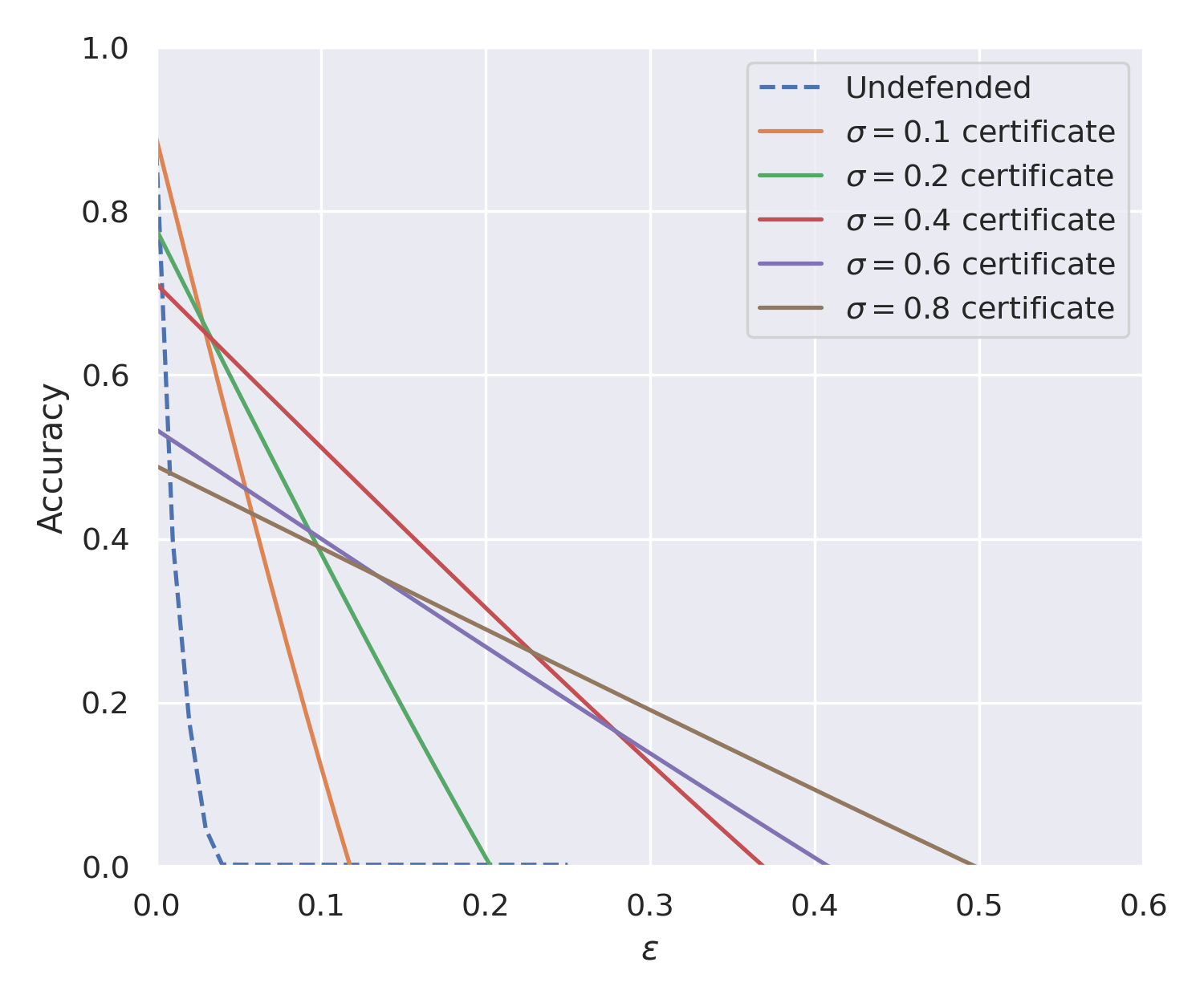}\label{fig:speech2}}\\
    \subfigure[Human activity recognition]{\includegraphics[width=0.9\linewidth]{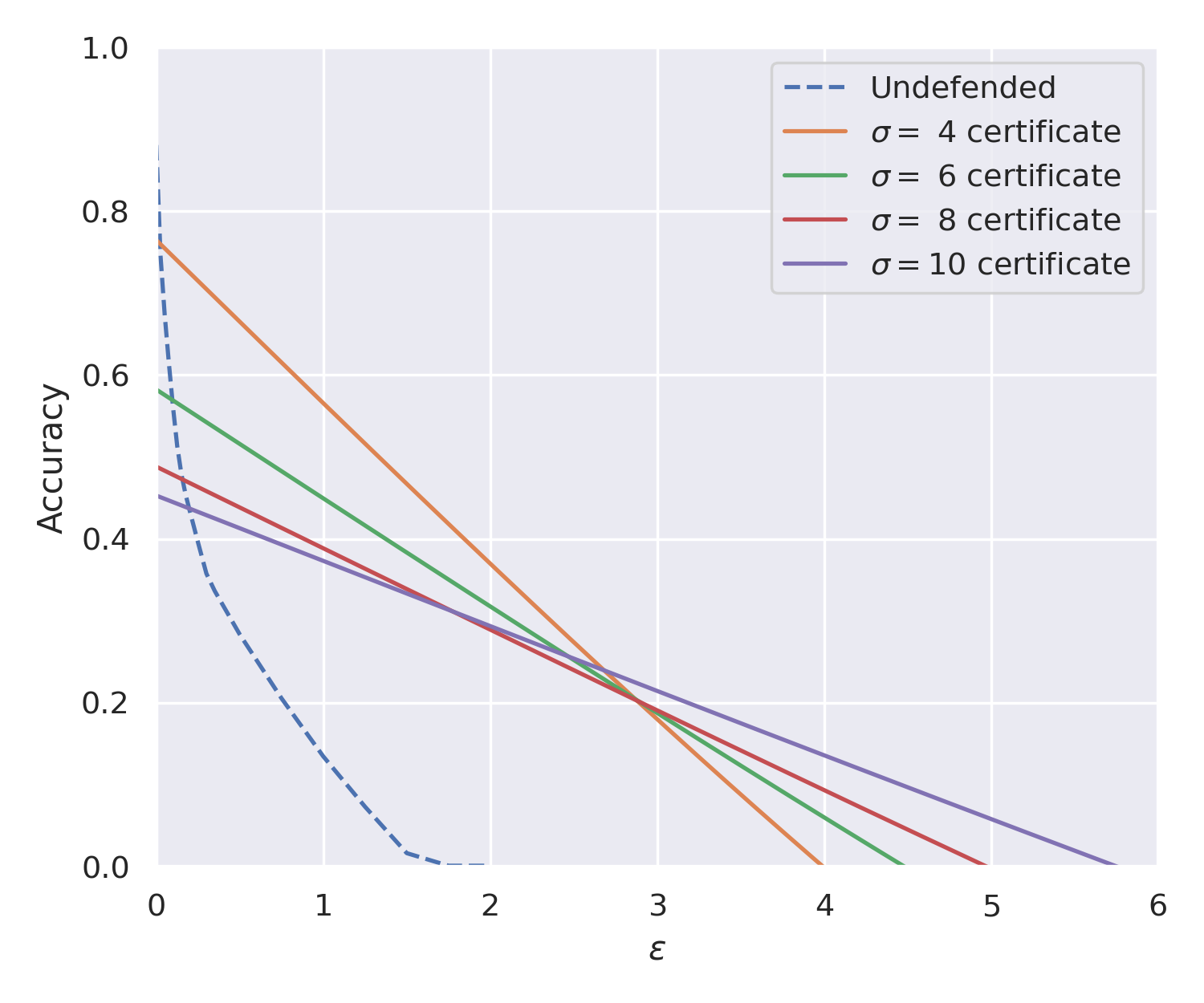}\label{fig:har2}}
        \caption{Certificates against online adversarial attacks for varying smoothing noises. Here we attack each window with the ability to re-attack inputs. The average size of perturbation is computed as per equation \ref{eq:adv_constraint_window}.}
        \label{fig:certificates2}
        \vspace{-6mm}
\end{figure}

We evaluate the robustness of undefended models using a custom-made attack that is constrained by the $\ell_2$-norm budget, as described in equation \ref{eq:adv_constraint}. To adhere to this constraint at each time-step $j$, the attacker must only perturb the input $x_j$, since the previous inputs $(x_{j-w+1},...,x_{j-1})$ have already been perturbed. This creates a significant challenge in creating a strong adversary. We design an adversary that only perturbs the last input $x_j$ at every time-step $j$ using projected gradient descent to minimize $f_j$. In our experiments, we set $f_j = 1$ if the model outputs the correct class and $f_j = 0$ when the model misclassifies. We linearly search using grid search parameter $\alpha$ for the smallest distance $d(x_j, x_j')$ such that the input $(x_{j-w+1}',...,x_j')$ leads to a misclassification at time-step $j$. We perturb $x_j$ if $(x_{j-w+1}',...,x_j')$ leads to misclassification and the average distance budget at time-step $j$ is less than $\epsilon$. Else, we do not perturb $x_j$. In this manner, our attack perturbs the streaming input in a greedy fashion. See Algorithm \ref{alg:onlineattack} for details.

We conduct our streaming attack on the keyword recognition task with a window size of $w=2$, where each input $x_j$ is a 4000-dimensional vector in the range [0,1]. We also perform the attack on the human activity recognition task with $w=2$, where each input $x_j$ is a 250x6-dimensional matrix. We use search parameter $\alpha = 15$. We plot the results of our certificates for various smoothing noises (see Figure \ref{fig:certificates1}). Note that the attack budget $\epsilon$ is calculated as per the definition in equation \ref{eq:adv_constraint}. In Figure \ref{fig:bestcertificates1}, we also plot our best certificates across various smoothing noises for different window sizes $w$. This plot supports our theory that streaming models with smaller window sizes are more robust to adversarial perturbations. Figures \ref{fig:attacksmoothhar} and \ref{fig:attacksmoothspeech} in Appendix \ref{app:attacksmooth} show that the empirical performance of smooth models after the online adversarial attack is lower bound by our certificates. These plots validate our certificates.

\subsection{Attacking each window}

Now, we perform experiments for the attack setting described in Section \ref{sec:attackingeahcwindow}. Note that here we need to calculate the attack budget $\epsilon$ based on equation \ref{eq:adv_constraint_window}. In this setting, we can re-attack an input for every window, making it a stronger attack. To attack the undefended models, we search for window perturbations that lead to misclassification using a minimum distance budget. Similar to our previous attack in Section \ref{sec:attackonce}, we only perturb a window at time-step $j$ if the average window distance at time-step $j$ is less than $\epsilon$. Also, we do not perturb a window if the window can not be perturbed to reduce the performance $f_j$. In Figure \ref{fig:certificates2}, we plot our certificates for this attack setting along with the accuracy of the undefended model for different attack budgets. These experiments show that our certificates produce meaningful performance guarantees against adversarial perturbations even if an attacker has the ability to re-attack the inputs. Figure \ref{fig:attacksmoothhar_2} in Appendix \ref{app:attacksmooth} shows that the empirical performance of smooth models after the online adversarial attack is lower bound by our certificates. These plots validate our certificates.